\documentclass{article}
\usepackage{authblk}
\usepackage{fullpage}
\usepackage[numbers, sort&compress]{natbib}

\usepackage[hyphens]{url}  
\usepackage{graphicx} 
\usepackage{caption} 
\usepackage{booktabs}
\usepackage{tcolorbox}

\usepackage{amsfonts}
\usepackage{latexsym} 
\usepackage{amsmath}
\usepackage{amsthm}
\usepackage{aliascnt}
\usepackage{mathtools}
\usepackage{mathabx}
\usepackage{bm}
\usepackage{mathrsfs}
\usepackage{xcolor}
\usepackage{xspace}
\DeclareMathAlphabet{\pazocal}{OMS}{zplm}{m}{n}
\DeclareMathAlphabet{\mathbbm}{U}{bbm}{m}{n}

\newcommand{\squishlist}{
\begin{list}{$\bullet$}
{  \setlength{\itemsep}{0pt}
\setlength{\parsep}{3pt}
\setlength{\topsep}{3pt}
\setlength{\partopsep}{0pt}
\setlength{\leftmargin}{2em}
\setlength{\labelwidth}{1.5em}
\setlength{\labelsep}{0.5em}
} }
\newcommand{\squishlisttight}{
\begin{list}{$\bullet$}
{ \setlength{\itemsep}{0pt}
	\setlength{\parsep}{0pt}
	\setlength{\topsep}{0pt}
	\setlength{\partopsep}{0pt}
	\setlength{\leftmargin}{2em}
	\setlength{\labelwidth}{1.5em}
	\setlength{\labelsep}{0.5em}
}}

\newcommand{\squishdesc}{
\begin{list}{}
	{  \setlength{\itemsep}{0pt}
		\setlength{\parsep}{3pt}
		\setlength{\topsep}{3pt}
		\setlength{\partopsep}{0pt}
		\setlength{\leftmargin}{1em}
		\setlength{\labelwidth}{1.5em}
		\setlength{\labelsep}{0.5em}
} }

\newcommand{\squishend}{
\end{list}
}

\newcommand{\Acal}{\mathcal{A}}
\newcommand{\Bcal}{\mathcal{B}}

\newcommand{\Ical}{\mathcal{I}}

\newaliascnt{proposition}{theorem}
\newtheorem{proposition}[proposition]{Proposition}
\aliascntresetthe{proposition}

\newaliascnt{corollary}{theorem}
\newtheorem{corollary}[corollary]{Corollary}
\aliascntresetthe{corollary}

\newaliascnt{remark}{theorem}
\newtheorem{remark}[remark]{Remark}
\aliascntresetthe{remark}

\newtheorem{mechanism}{Interaction Mechanism}
\newtheorem{definition}{Definition}

\newaliascnt{lemma}{theorem}

\aliascntresetthe{lemma}

\newcommand{\eg}{\emph{e.g.}\xspace}
\newcommand{\ie}{\emph{i.e.}\xspace}

\newcommand{\xhdr}[1]{\vspace{0.1cm}\noindent {\bf #1.}}

\usepackage{thmtools}
\usepackage{thm-restate}

\usepackage{cleveref}

\Crefname{theorem}{theorem}{theorems}
\Crefname{theorem}{Theorem}{Theorems}
\Crefname{section}{section}{sections}
\Crefname{section}{Section}{Sections}
\Crefname{lemma}{Lemma}{Lemmas}
\Crefname{proposition}{proposition}{propositions}
\Crefname{proposition}{Proposition}{Propositions}
\Crefname{corollary}{corollary}{corollaries}
\Crefname{corollary}{Corollary}{Corollaries}
\Crefname{remark}{remark}{remarks}
\Crefname{remark}{Remark}{Remarks}
\Crefname{problem}{problem}{problems}
\Crefname{problem}{Problem}{Problems}
\Crefname{mechanism}{mechanism}{mechanisms}
\Crefname{mechanism}{Mechanism}{Mechanisms}
\Crefname{model}{model}{models}
\Crefname{model}{Model}{Models}
\Crefname{assumption}{Assumption}{assumptions}
\Crefname{assumption}{assumption}{Assumptions}

\newcommand{\PoM}{\ensuremath{\mathrm{PoM}}\xspace}

\title{Linguistic Monoculture in LLM-Assisted Language Use}
\author[$\S$]{Suhas Thejaswi}
\author[$\S$]{Juhi Kulshreshta}
\author[$\dagger$]{Lutz Oettershagen}

\affil[$\S$]{Aalto University, Finland\protect\\ 
\texttt{firstname.lastname@aalto.fi} \vspace{0.5em}}

\affil[$\dagger$]{University of Liverpool, United Kingdom \protect\\
\texttt{firstname.lastname@liverpool.ac.uk} \vspace{0.5em}}

\date{}

\begin{document}

\maketitle
\begin{abstract}
    Writing and communication are increasingly mediated by large language models (LLMs) that are being used to draft, revise and polish text. Although such assistance can improve clarity and help authors meet institutional expectations, widespread reliance on shared models may reduce population-level variation in linguistic form, a phenomenon we refer to as \emph{linguistic monoculture}. We develop a mathematical framework in which authors and LLMs are represented as distributions over linguistic features and co-evolve through repeated interaction. We analyze three interaction mechanisms: a shared model with a fixed linguistic distribution, a shared model recursively updated from author outputs, and personalized models updated through author-specific and population-level feedback. We characterize the resulting equilibria and convergence rates, showing that, shared models can drive authors toward a common norm, recursive feedback relocates the shared norm without altering pairwise spread under common conformity, and personalization can preserve a family of distinct author–model equilibria with nonzero linguistic diversity. We then endogenize conformity as a strategic choice trading off private benefits from clarity, legibility, and perceived fluency against distinctive style. Within this utility model, individually rational authors may conform more than is socially optimal because they do not internalize the value their distinctiveness provides to others, creating a negative externality and a price of monoculture that is finite for each fixed instance but can grow without bound when distinctiveness dominates authenticity. Synthetic simulations illustrate how fixed shared assistance, recursive feedback, and personalization produce different long-run diversity outcomes.
\end{abstract}

\paragraph{Keywords:} Algorithmic monoculture, Linguistic monoculture, Price of Monoculture

\section{Introduction} \label{sec:introduction}
Large language models (LLMs) are increasingly used to draft, revise, and polish text. Although such assistance can improve clarity, reduce errors, and help authors meet institutional expectations, widespread reliance on shared models may pull users toward similar model-mediated linguistic patterns. Moreover, LLM-generated or LLM-assisted text may enter future training and personalization pipelines, creating a coupled population-level system in which authors adapt to model outputs and models may in turn adapt to language already influenced by earlier models. Repeated interaction with shared generative systems may therefore reduce variation in language through which people develop, communicate and distinguish ideas~\citep{sourati2025shrinking,padmakumar2024does}.

We call this reduction in population-level variation across lexical choices, syntactic constructions, discourse markers, and other stylistic patterns as {\emph{linguistic monoculture}}. Recent empirical work provides evidence that LLM assistance can homogenize human text, ideas, expressions, and stylistic choices across users~\citep{sourati2026homogenizing,padmakumar2024does,doshi2024generative,anderson2024homogenization,moon2025homogenizing,jakesch2023co}. Large-scale studies have also documented shifts in word frequencies and LLM-associated stylistic markers in scientific articles and abstracts~\citep{kobak2025delving,geng2025impact,geng2025human,liang2025quantifying}. This concern is especially consequential in academia, where scientific writing helps make concepts legible, establish accepted methods, and shape what intellectual communities recognize as a contribution~\citep{bazerman1988shaping,swales1990genre,hyland2009academic,latour2013laboratory}. We therefore ask:\looseness=-1
\begin{tcolorbox}
\textbf{\emph{Under what forms of repeated author–model interaction does LLM assistance drive a population toward a shared linguistic norm, and when can heterogeneous preferences or personalization preserve diversity?}}
\end{tcolorbox}

Our analysis focuses on diversity in linguistic-feature distributions rather than convergence in semantic content, reasoning, or intellectual perspective. Linguistic variation is nevertheless a distinct and measurable dimension of expression through which authors establish voice, structure arguments and make distinctions legible to others~\citep{gumperz1996rethinking,lucy1997linguistic,evans2009myth,enfield2015linguistic}. At the same time, not all convergence is undesirable: shared conventions can improve clarity, reduce errors, and make writing easier to understand and evaluate~\citep{hyland2009academic,bietti2026will}. The relevant question is therefore when shared AI assistance produces useful standardization and when it causes an excessive loss of population-level linguistic diversity.

To answer this question, we develop a reduced-form framework representing authors and prompt-conditioned LLM output distributions over linguistic features. It isolates how sharing, deployment-level feedback, and personalization affect population-level diversity, measured by average pairwise Jensen--Shannon (JS) divergence between author distributions.
We examine three author–LLM interaction mechanisms: a shared model with a fixed linguistic distribution, a shared model updated recursively from author outputs, and personalized models updated through author-specific and population-level feedback. These mechanisms isolate whether the model is shared, whether it is recursively updated, and whether its feedback is personalized. We then endogenize conformity to a shared linguistic norm as a strategic choice and compare individually optimal conformity with the socially optimal level and quantify the loss of linguistic diversity through the \emph{price of monoculture}.

Our analysis of interaction dynamics initially treats conformity as exogenous. In practice, however, adopting a shared linguistic norm may be a strategic choice: standardized language can improve clarity, legibility and perceived fluency, and may be rewarded by reviewers, readers, or institutions~\citep{giles2016communication,bell1984language}. Yet such private benefits may come at the cost of an author's distinctiveness preferences and reduce population-level linguistic diversity. We therefore compare individually optimal conformity with the socially optimal level.
In detail, our contributions are as follows:\footnote{All proofs and omitted details are in the Appendix.}

\begin{enumerate}
\item \textbf{A reduced-form framework for LLM-mediated dynamics.}
We formalize authors and prompt-conditioned LLM output distributions over linguistic features under three deployment mechanisms: a fixed shared system, a recursively updated shared system, and recursively updated personalized systems (\Cref{sec:framework}).

\item \textbf{Linguistic convergence and determinants of diversity.}
We establish convergence and convergence-rate bounds for all three mechanisms and characterize how conformity pressure, author-specific preferences, recursive feedback, and personalization determine long-run population-level diversity. Shared models can drive diversity to low levels, whereas author-specific preferences and personalization can preserve diversity (\Cref{sec:convergence}).

\item \textbf{Strategic conformity and the price of monoculture.}
We show that individually rational conformity can exceed the social optimum because authors may not internalize the value their distinctiveness provides to others; we bound the price of monoculture in a symmetric regime and identify when it diverges (\Cref{sec:strategic}).

\item \textbf{Quantitative comparison of interaction mechanisms.}
Synthetic simulations illustrate how fixed shared assistance, recursive feedback, and personalization produce different long-run diversity outcomes across controlled parameter regimes. Paired runs on identical initializations enable controlled comparison of the interaction mechanisms (\Cref{sec:experiments} and Appendix~\ref{app:experiments}).
\end{enumerate}

\xhdr{Related work}
Our work connects research on LLM-induced homogenization, opinion dynamics, recursive model training, algorithmic monoculture, and communication accommodation. Building on classical averaging models, we couple author adaptation to shared or personalized models updated from author outputs, and analyze linguistic-diversity equilibria and strategic welfare.

\smallskip
\noindent
\emph{LLM-induced homogenization.}
Empirical studies find that LLM assistance can homogenize written expression, ideas, authorial voice, and cultural style across users~\citep{padmakumar2024does,anderson2024homogenization,moon2025homogenizing,abdulhai2026llms,agarwal2025ai}. Analyses of academic writing and presentations similarly document increased use of LLM-associated linguistic markers~\citep{geng2025impact,geng2025human,kobak2025delving,liang2025quantifying}. These studies motivate our problem; we complement them by formalizing how homogenization can emerge through repeated author–LLM interaction.

\smallskip
\noindent
\emph{Opinion dynamics.}
Our work builds on classical averaging models~\citep{degroot1974reaching,friedkin1990social,rainer2002opinion}, using them as a reduced-form representation of an algorithmic intermediary whose output distribution may be updated from the population it influences. This introduces two deployment choices absent from human-only models: whether the intermediary is shared or personalized, and whether author outputs affect its future behavior. These choices determine whether the system converges to one shared norm or to a family of author-specific equilibria.\looseness=-1

\smallskip
\noindent
\emph{Model collapse, knowledge collapse, and recursive feedback.}
Recursive training on model-generated data can degrade generative models and reduce coverage of the tails of the original data distribution~\citep{shumailov2024ai}. Related work studies how AI-mediated information environments may narrow the range of available knowledge and loss of tails~\citep{peterson2025ai,wright2025epistemic}. This literature centers on the quality or epistemic diversity of model outputs, whereas we study the linguistic diversity of human authors who adapt to and may shape a shared model.

\smallskip
\noindent
\emph{Algorithmic monoculture.}
\citet{kleinberg2021algorithmic} show that reliance on common algorithmic advice can be individually beneficial while reducing social welfare. Recent work by \citet{kleinberg2026price} quantifies this loss in matching markets, proving a tight price-of-anarchy bound of factor $2$. Our work studies a linguistic analogue: shared LLM assistance may improve individual legibility while reducing population-level linguistic diversity.

\noindent
\emph{Communication accommodation and personalization.}
Communication accommodation theory studies how speakers adjust their language toward or away from interlocutors~\citep{bell1984language,giles2016communication}, including in human–LLM interaction~\citep{chang2025aaai-communication}. Related language-game models examine how shared conventions emerge through peer-to-peer interaction~\citep{chen2019naming,de2006how}. In contrast, we study an LLM as a shared or personalized algorithmic interlocutor and characterize when personalization can preserve diversity across authors.

\section{A Mathematical Framework for LLM-Assisted Language Use}\label{sec:framework}
We use $[j]$ to denote $\{1,\ldots,j\}$ for $j\in\mathbb{N}$, and $\Delta^{m-1} := \{ x \in \mathbb{R}_{\ge 0}^m : \sum_{k=1}^m x_k = 1 \}$ to denote an $(m-1)$-dimensional probability simplex.
Let $\mathcal{F}=[m]$ be a common set of linguistic features encoding lexical choices, syntactic constructions, discourse markers or other stylistic patterns. We represent linguistic style as a probability distribution over $\mathcal{F}$. For each author $i\in[n]$ and time $t\in\{0,\ldots,T\}$, let $p_i^t\in\Delta^{m-1}$ denote author $i$'s linguistic-style distribution, with initial distribution $p_i^0$. 
Let $q^t\in\Delta^{m-1}$ denote the prompt-conditioned linguistic-style distribution induced by an LLM system at time $t$ under a fixed class (or distribution) of writing prompts; if personalized to author $i$, we write $q_i^t$. This is an output-level abstraction rather than a model-parameter representation. For notational uniformity, in shared-model settings we write $q_i^t:=q^t$ for every author $i$.

At each time step $t\in\{0,\ldots,T-1\}$, author $i$ prompts the LLM,
observes LLM-generated linguistic suggestions and updates their own style as
\begin{equation} \label{eq:author_update}
	p_i^{t+1} = \underbrace{(1-\alpha_i)\, p_i^t}_{\substack{\text{Style}\\ \text{retention}}}
	+ \underbrace{\alpha_i\, \Acal_i(q_i^t,z_i^t)}_{\substack{\text{Model- and}\\ \text{context-driven update}}},
\end{equation}
where $\alpha_i\in[0,1]$ measures the rate at which author $i$ adapts their linguistic style. The operator $\Acal_i:\Delta^{m-1}\times\mathcal Z\to\Delta^{m-1}$ describes how the author incorporates the model's output given an author-specific adaptation input $z_i^t\in\mathcal Z$, which may encode external incentives, preferred style, or other contextual information.

The induced output distribution may itself evolve when author-produced text is incorporated into future training data, a recursive feedback mechanism studied in the model-collapse literature~\citep{shumailov2024ai}. We model this in reduced form as
\begin{equation} \label{eq:model_update}
	q_i^{t+1} =  \underbrace{\beta q_i^t}_{\substack{\text{Model}\\ \text{retention}}} +
	\underbrace{(1-\beta)\, \Bcal(p_1^{t+1},\ldots,p_n^{t+1})}_{\substack{\text{Feedback from}\\ \text{author outputs}}},
\end{equation}
where $\beta \in [0,1]$ is the retention parameter: larger $\beta$ places more weight on the past model distribution. The operator $\Bcal$ maps the updated author distributions to the distribution of training data used to update the model; we instantiate this to be the weighted population average $\Bcal(p_1^{t+1},\ldots,p_n^{t+1})= \sum_{i=1}^n w_i p_i^{t+1}$ with $w_i\ge 0$, $\sum_i w_i = 1$.

Our main object of interest is population-level linguistic diversity, measured by the average pairwise Jensen--Shannon divergence~\citep{lin1991divergence} among author distributions:
\begin{align*}
	D^t &= \frac{1}{n(n-1)} \sum_{i\neq j} JS(p_i^t,p_j^t),
\end{align*}
where $JS(p,q) = \frac{1}{2}KL(p\|x) + \frac{1}{2}KL(q\|x)$ with $x=\frac{p+q}{2}$; smaller values indicate convergence toward a common linguistic norm. When the model distribution evolves, we also track the average author--model divergence
\[
M^t = \frac{1}{n}\sum_{i=1}^n JS(p_i^t,q_i^t),
\]
which measures the alignment between each author and the model they interact with; in shared-model settings, $M^t\to 0$ corresponds to a shared human-LLM linguistic equilibrium. Finally, when each author interacts with a personalized model, we track diversity among the personalized models,
\[
Q^t = \frac{1}{n(n-1)} \sum_{i\neq j} JS(q_i^t,q_j^t),
\]
where $Q^t\to 0$ means all personalized model distributions become indistinguishable. In the personalized setting, $D^t$ measures diversity among authors, $Q^t$ diversity among personalized models and $M^t$ alignment between each author and their model. Since Jensen-Shannon divergence is symmetric and finite for probability distributions, it is a natural choice for all three measures. \Cref{rem:distributional} in Appendix~\ref{app:remarks} justifies our choice to represent linguistic style as a distribution rather than a point in a high-dimensional feature space.

\medskip
\subsection{Author-LLM interaction mechanisms}
We formalize three interaction mechanisms (IMs) that introduce increasing degrees of author--model coupling.

\begin{mechanism}[Shared model with fixed distribution]
	\label{problem1}
	All authors interact with the same model, whose linguistic-style distribution remains fixed over time, \ie, $q^t=q^0$ for all $t\in\{0,\ldots,T\}$. The linguistic-style distribution of author $i$ evolves as
	\[
	p_i^{t+1} = (1-\alpha_i)\,p_i^t + \alpha_i\,\Acal_i(q^0,z_i^t).
	\]
\end{mechanism}

\begin{mechanism}[Shared model with recursive updates]
	\label{problem2}
	All authors interact with the same model, but the model distribution evolves over time. The coupled author--model dynamics evolve as
	\begin{align*}
		p_i^{t+1} &= (1-\alpha_i)\,p_i^t + \alpha_i\, \Acal_i(q^t,z_i^t),\\
		q^{t+1} &= \beta\,q^t + (1-\beta)\,\Bcal(p_1^{t+1},\ldots,p_n^{t+1}),
	\end{align*}
	with $\Bcal(p_1^{t+1},\ldots,p_n^{t+1}) = \sum_{j=1}^n w_j p_j^{t+1}$ as the weighted-average instantiation, where $w_j\geq 0$ and $\sum_{j=1}^n w_j=1$. 
\end{mechanism}

\begin{mechanism}[Personalized models with recursive updates]
	\label{problem3}
	Each author $i$ interacts with an author-specific personalized model with linguistic-style distribution $q_i^t$, initialized with a common base distribution $q_i^0=q^0$ for all $i\in[n]$. The coupled author--model dynamics evolve as
	\begin{align*}
		p_i^{t+1} &= (1-\alpha_i)\,p_i^t+\alpha_i\,\Acal_i(q_i^t,z_i^t), \\
		q_i^{t+1} &= \beta\,q_i^t+(1-\beta)\,\Bcal_i(p_1^{t+1},\ldots,p_n^{t+1}),
	\end{align*}
	where the feedback operator $\Bcal_i$ mixes author $i$'s own updated distribution with the population-level update:
	\[
	\Bcal_i(p_1^{t+1},\ldots,p_n^{t+1}) = \rho\,p_i^{t+1} + (1-\rho)\sum_{j=1}^n w_j p_j^{t+1},
	\]
	with $\rho\in[0,1]$, $w_j\geq 0$, $\sum_{j=1}^n w_j=1$. Equivalently,
	\[ q_i^{t+1} = \beta \,q_i^t + \gamma \,p_i^{t+1} + \delta \sum_{j=1}^n w_j p_j^{t+1},
	\]
	where $\gamma=(1-\beta)\rho$, $\delta=(1-\beta)(1-\rho)$, and $\beta+\gamma+\delta=1$. The parameter $\rho$ determines the degree of personalization: when $\rho=1$, the model update depends only on author $i$'s own distribution; when $\rho=0$, all personalized models receive the same population-level feedback.
\end{mechanism}

In Appendix~\ref{app:im4}, we investigate Interaction Mechanism~4, which partitions a single population into three subpopulations, each following the update rule of one of IMs~1-3.

\section{Convergence of Linguistic Distributions}\label{sec:convergence}
We analyze the three interaction mechanisms in increasing order of coupling, establishing convergence as well as convergence-rate bounds. 

\subsection{A shared model with fixed distribution}
In \Cref{problem1} the model distribution is fixed, so $q^t=q^0$ for all $t$ and author $i$'s linguistic-style distribution evolves as
$p_i^{t+1} = (1-\alpha_i)\, p_i^t + \alpha_i\, \Acal_i(q^0,z_i^t)$.
First, we consider the model-induced adaptation to be homogeneous across authors and time, $\Acal_i(q^0,z_i^{t+1})=\Acal_i(q^0,z_i^t)=\Acal(q^0,z)$ for all $i$ and $t$, so all authors are influenced by the same fixed model distribution and have the same adaptation input.

\begin{restatable}[]{proposition}{proplinguisticconvergence} \label{prop:linguistic_convergence}
	In \Cref{problem1}, suppose the model-induced adaptation is homogeneous across authors and time, \ie,
	$\Acal_i(q^0,z_i^{t})=\Acal(q^0,z)$
	for every author $i$ and time $t$. Let
	$\alpha_{\min} = \min_{i \in [n]}\alpha_i>0$
	be the minimum adaptation rate. Then, for every $\varepsilon > 0$, it is sufficient to take
	$t \geq \frac{\log(2/\varepsilon)}{\alpha_{\min}}$ to guarantee $D^t \leq \varepsilon$.
\end{restatable}

The homogeneity assumption is strong: all authors are pulled toward the same distribution. Thus \Cref{prop:linguistic_convergence} gives a baseline showing that common adaptation and rewards cause diversity to decay exponentially, at a rate set by the slowest-adapting author. Our next result relaxes the homogeneity assumption with author-specific adaptation.

\begin{restatable}[]{proposition}{propauthorspecificfixedadaptation} \label{prop:author_specific_fixed_adaptation}
	Consider \Cref{problem1} with fixed model distribution $q^t=q^0$. Suppose that the adaptation operator is time-invariant but author-specific, \ie, for each author $i$, $\Acal_i(q^0,z_i^t)=a_i$ for all $t$, where
	$a_i\in\Delta^{m-1}$, and let $\alpha_{\min}=\min_{i\in[n]}\alpha_i>0$.
	Then each author's linguistic-style distribution converges to its author-specific adapted distribution, $p_i^t \to a_i$, exponentially fast at a rate controlled by $\alpha_{\min}$. Consequently,
	\[
	D^t \to D^\infty := \frac{1}{n(n-1)} \sum_{i\neq j} JS(a_i,a_j),
	\]
	and \(D^t\to 0\) if and only if \(a_i=a_j\) for all \(i,j\).
\end{restatable}

When adaptation is author-specific but fixed over time, authors do not collapse to a single linguistic-style distribution: each converges to their own adaptation target, and population-level diversity converges to the diversity among these targets---author-specific incentives preserve diversity even under repeated exposure to the same fixed model. If $D^\infty>0$, the process stabilizes at a positive level of diversity.\looseness=-1

\xhdr{Interplay between diversity and conformity}
To interpolate between these regimes, let each author balance a conformity incentive pulling toward the shared model-induced norm against a diversity incentive preserving author-specific style.
For the remainder of this section, let $r_i\in\Delta^{m-1}$ denote author $i$'s fixed preferred linguistic-style distribution, and set $z_i^t=r_i$ for all $t$.
We model the adaptation as
\[
\Acal_i(q^0,r_i) = (1-\lambda_i)r_i+\lambda_i q^0, \qquad \lambda_i\in[0,1],
\]
where $\lambda_i$ measures the strength of the conformity incentive: $\lambda_i=0$ preserves the preferred style $r_i$, while $\lambda_i=1$ fully adopts the shared model distribution $q^0$.

\begin{restatable}[]{proposition}{propfixedconformitymixture} \label{prop:fixed_conformity_mixture}
	Consider \Cref{problem1} with fixed model distribution $q^t=q^0$. Suppose each author has an author-specific preferred distribution $r_i\in\Delta^{m-1}$, and the adaptation operator is
	\[ \Acal_i(q^0,r_i) = (1-\lambda_i)r_i+\lambda_i q^0, \qquad \lambda_i\in[0,1].
	\]
	Then each author distribution converges to
	$a_i(\lambda_i) := (1-\lambda_i)\,r_i+\lambda_i q^0$. Precisely, if $\alpha_{\min} :=\min_{i \in [n]} \alpha_i > 0$
	then for every $\varepsilon>0$, it is sufficient to take $t \geq \frac{\log(2/\varepsilon)}{\alpha_{\min}}$ to guarantee that
	$
	\max_{i \in [n]} \|p_i^t - a_i(\lambda_i)\|_1 \leq \varepsilon.
	$
	Consequently,
	\[
	D^t \to D^\infty(\lambda) := \frac{1}{n(n-1)} \sum_{i\neq j}
	JS\left(a_i(\lambda_i),a_j(\lambda_j)\right).
	\]
	In particular, $D^t\to 0$ if and only if $a_i(\lambda_i)=a_j(\lambda_j)$ for all $i,j$. A sufficient condition is $\lambda_i=1$ for all $i$, in which case $a_i(\lambda_i)=q^0$ for every author $i$.
\end{restatable}

The proof establishes that for a common conformity level $\lambda_i=\lambda$, the limiting diversity satisfies $D^\infty(\lambda)=O(1-\lambda)$.
The parameter $\lambda$ represents institutional or reward-based pressure to conform to a shared model-induced linguistic norm: when $\lambda$ is small, author-specific preferences dominate and diversity persists; when $\lambda$ is large, the shared norm dominates and diversity collapses. Even with author-specific adaptation, monoculture can emerge if the reward structure places sufficiently high weight on conformity.

\subsection{A shared model with recursive updates}
We next allow the shared model to be retrained on author outputs. Writing $P^{t+1}:=\Bcal(p_1^{t+1},\ldots,p_n^{t+1}) = \sum_{i=1}^n w_i p_i^{t+1}$, the coupled dynamics of \Cref{problem2} are
\begin{align*}
	p_i^{t+1} &= (1-\alpha_i)\, p_i^t + \alpha_i\, \Acal_i(q^t,z_i^t), \\
	q^{t+1} &= \beta q^t + (1-\beta)\, P^{t+1}.
\end{align*}
If the adaptation target does not depend on the evolving model distribution ($\Acal_i(q^t,z_i^t)=a$ or $a_i$), the author dynamics are unchanged from the fixed-model setting and \Cref{prop:linguistic_convergence,prop:author_specific_fixed_adaptation} apply verbatim: recursive feedback alone does not force monoculture when authors are pulled toward fixed targets. Recursive updates matter only when the adaptation operator depends non-trivially on $q^t$. We therefore consider the conformity-mixture adaptation
\[
\Acal_i(q^t,r_i)=(1-\lambda_i)r_i+\lambda_i q^t,
\qquad \lambda_i\in[0,1],
\]
where the author's preferred distribution $r_i$ is fixed over time and $\lambda_i$ now measures conformity to the \emph{current} model distribution.

\begin{restatable}[]{proposition}{propsharedrecursiveconformity} \label{prop:shared_recursive_conformity}
	Consider \Cref{problem2} with
	$\Acal_i(q^t,r_i)=(1-\lambda_i)r_i+\lambda_i q^t$ and $\lambda_i\in[0,1)$.
	Suppose that $\alpha_i\in(0,1]$ for every author $i$, $\beta\in[0,1)$, and
	$P^{t+1}=\sum_{i=1}^n w_i p_i^{t+1}$, $w_i\geq 0$, $\sum_{i=1}^n w_i=1$.
	Then the coupled author--model dynamics converge to an equilibrium
	$(p_1^\ast,\ldots,p_n^\ast,q^\ast)$ satisfying
	\[
	p_i^\ast=(1-\lambda_i)r_i+\lambda_i q^\ast
	\]
	\[
	\text{for every author $i$, and} \quad
	q^\ast = \frac{\sum_{i=1}^n w_i(1-\lambda_i)r_i} {1-\sum_{i=1}^n w_i\lambda_i}.
	\]
	\[
	\text{Consequently,} \quad
	D^t\to D^\ast := \frac{1}{n(n-1)} \sum_{i\neq j} JS(p_i^\ast,p_j^\ast).
	\]
	In particular, the limiting diversity is the diversity among the equilibrium
	author distributions $p_i^\ast$, rather than the diversity among the fixed-model
	targets $(1-\lambda_i)\,r_i+\lambda_i\,q^0$.
\end{restatable}

The proof also shows that convergence is exponential, with a worst-case time bound of
$t \geq \frac{\log(2/\varepsilon)}{(1-\beta)\min_{i \in [n]} \alpha_i(1-\lambda_i)}$,
which is weaker than the fixed-model bound, controlled by $\alpha_{\min}$ alone: the moving target $q^t$ can slow the guaranteed rate, though the bound need not be tight in typical instances (see \Cref{sec:experiments}). The condition $\lambda_i<1$ ensures that each author retains some pull toward their preferred distribution $r_i$; if $\lambda_i=1$ for all authors, the fixed-point formula for $q^\ast$ becomes singular and that case must be treated separately.

The weaker rate bound should not be read as recursive feedback dampening homogenization: recursion changes the equilibrium, anchoring authors at the endogenous $q^\ast$ rather than the exogenous $q^0$. \Cref{prop:common_conformity} makes the resulting comparison precise.

\subsection{Personalized models with recursive updates}
Personalization changes what the system converges \emph{to}: because each model update includes an author-specific feedback component, the limiting model distributions $q_i^\ast$ can remain distinct. The system may therefore converge to a family of author-specific equilibria rather than a single shared norm. The parameter $\rho$ controls the extent of personalization.

\begin{restatable}[]{proposition}{proppersonalizedrecursive} \label{prop:personalized_recursive}
	Consider \Cref{problem3} with author and model updates
	\begin{align*}
		p_i^{t+1} &= (1-\alpha_i)\,p_i^t+\alpha_i\,\Acal_i(q_i^t,r_i), \\
		q_i^{t+1} &= \beta q_i^t+\gamma p_i^{t+1} +\delta\sum_{j=1}^n w_jp_j^{t+1},
	\end{align*}
	where $\alpha_i\in(0,1]$, $\Acal_i(q_i^t,r_i)=(1-\lambda_i)r_i+\lambda_i q_i^t$ with $\lambda_i\in[0,1)$, $w_j\geq 0$, $\sum_{j=1}^n w_j=1$, and $\beta+\gamma+\delta=1$.
	Suppose that $\beta\in[0,1)$, and let
	$\rho=\frac{\gamma}{1-\beta}\in[0,1]$. Then the coupled personalized author--model dynamics converge to an equilibrium
	$(p_1^\ast,\ldots,p_n^\ast,q_1^\ast,\ldots,q_n^\ast)$ with
	\[
	p_i^\ast=\eta_i r_i+(1-\eta_i)P^\ast, \quad q_i^\ast=\rho\eta_i r_i+(1-\rho\eta_i)P^\ast,
	\]
	where $\eta_i=\frac{1-\lambda_i}{1-\rho\lambda_i}$ and
	$P^\ast = \frac{\sum_{i=1}^n w_i\eta_i r_i}{\sum_{i=1}^n w_i\eta_i}$.
	Consequently,
	\[
	D^t\to D^\ast:= \frac{1}{n(n-1)} \sum_{i\neq j}JS(p_i^\ast,p_j^\ast),
	\]
	\[Q^t\to Q^\ast:= \frac{1}{n(n-1)} \sum_{i\neq j}JS(q_i^\ast,q_j^\ast).
	\]
\end{restatable}

When $\rho>0$, each equilibrium model distribution $q_i^\ast$ retains an author-specific component proportional to $\rho\eta_i$, so both author diversity $D^\ast$ and personalized-model diversity $Q^\ast$ can remain bounded away from zero; when $\rho=0$, all personalized models collapse to the same population-level distribution. Personalization thus does not prevent convergence---it changes its object, from a single shared norm to a family of distinct author--model equilibria.

Because JS depends on both pairwise separation and the cloud's location in the simplex (\Cref{rem:quadratic}), we isolate pairwise geometry using the translation-invariant quadratic diversity
\[
\widehat{D}(p_1,\ldots,p_n):=\frac{1}{2n(n-1)}\sum_{i\neq j}\|p_i-p_j\|_2^2,
\]
and write $\widehat{D}_\infty$ for its value at equilibrium.

\begin{restatable}[]{proposition}{propcommonconformity} \label{prop:common_conformity}
    Suppose $\lambda_i=\lambda\in[0,1)$ for every author $i$, and let
    $\eta=\frac{1-\lambda}{1-\rho\lambda}$. Then, for every pair $i\neq j$,
    \[
    a_i(\lambda)-a_j(\lambda)=(1-\lambda)(r_i-r_j)
    \quad\text{under IM~1},
    \]
    \[
    p_i^\ast-p_j^\ast=(1-\lambda)(r_i-r_j)\quad\text{under IM~2},\]
    \[
    p_i^\ast-p_j^\ast=\eta(r_i-r_j)\quad\text{under IM~3}.
    \]
    Consequently, IMs~1 and~2 have identical limiting quadratic diversity, while
    \[
    \widehat{D}_\infty(\mathrm{IM~3})
    =\frac{\widehat{D}_\infty(\mathrm{IM~1})}{(1-\rho\lambda)^2}
    \geq\widehat{D}_\infty(\mathrm{IM~1}),
    \]
    with equality if and only if $\rho\lambda=0$.
\end{restatable}

Under common conformity, recursion relocates the limiting author cloud from the exogenous $q^0$ to the endogenous $q^\ast$ without changing its pairwise geometry. Hence any IM~1--IM~2 gap in limiting JS diversity arises only from JS's position dependence. Personalization instead expands pairwise geometry relative to IMs~1--2 by $1/(1-\rho\lambda)$; heterogeneous $\lambda_i$ can additionally create a genuine IM~1--IM~2 geometric gap (Appendix~\ref{app:experiments}).

\xhdr{Summary} These results treat the conformity parameters $\lambda_i$ as exogenous: they characterize what happens when authors have a fixed tendency to adopt the model-induced norm, but not why authors would choose to have that tendency. 

\section{Strategic Conformity and the Price of Monoculture}\label{sec:strategic}
In \Cref{sec:convergence}, the conformity parameter $\lambda_i$ quantifies \emph{how much} author $i$ is influenced by the model-induced linguistic norm, but not \emph{why}. We now endogenize this parameter in the setting of \Cref{prop:fixed_conformity_mixture}: each author chooses $\lambda_i$ to trade off individual rewards of conforming---such as legibility, perceived polish and alignment with reviewer expectations---against the loss of distinctiveness that conformity entails. This turns the dynamics into a strategic game, and we ask the linguistic analogue of the question raised by algorithmic monoculture~\citep{kleinberg2021algorithmic}: \emph{can individually rational linguistic adaptation produce collectively inefficient homogenization?}
Under the fixed-model quadratic game with orthogonal signatures analyzed below, the answer is yes: the unique Nash equilibrium conformity level weakly exceeds the socially optimal level for every author, and strictly exceeds it for any author who conforms when others value distinctiveness. This gap is the value of an author's distinctiveness to others: a payoff externality no individual internalizes and one that harms even authors who optimally choose not to conform (\Cref{lem:externality}). In a symmetric regime, the resulting price of monoculture can be arbitrarily large across instances.

\subsection{Conformity as a choice} \label{subsec:conformity-choice}
We work in the setting of \Cref{problem1}, \Cref{prop:fixed_conformity_mixture}: the model distribution is fixed at $q^0$, each author $i$ has a preferred distribution $r_i \in \Delta^{m-1}$ and the adaptation operator is
$\Acal_i(q^0, r_i) = (1-\lambda_i)\, r_i + \lambda_i\, q^0$. 
For any conformity profile $\lambda = (\lambda_1,\dots,\lambda_n) \in [0,1]^n$, \Cref{prop:fixed_conformity_mixture} shows that the author dynamics converge exponentially to
$a_i(\lambda_i) = (1-\lambda_i)\, r_i + \lambda_i\, q^0$.
We use this convergence result to define a strategic game over long-run writing policies. Each author first chooses a fixed conformity level $\lambda_i$, which determines how strongly they rely on the model-induced norm. The interaction dynamics in \cref{eq:author_update} then unfold and payoffs are evaluated at the limiting profile $(a_1(\lambda_1),\dots,a_n(\lambda_n))$. Thus the game is played over long-run writing styles and not over individual time-step updates.
This timescale separation is justified by exponential convergence: when the evaluation horizon is long relative to the convergence time, authors' distributions spend most of the horizon close to their limiting profile, so payoffs evaluated at the limiting profile approximate long-run average payoffs.

For brevity, we write $u_i := r_i - q^0$ for author $i$'s \emph{signature vector}---the direction and magnitude of their idiosyncrasy relative to the model's linguistic norm---and $\sigma_i := 1-\lambda_i \in [0,1]$ for their \emph{retained distinctiveness}, so that
\[
a_i(\lambda_i) - q^0 = \sigma_i\, u_i, \quad
a_i(\lambda_i) - r_i = -\lambda_i\, u_i,
\]
\[
a_i(\lambda_i) - a_j(\lambda_j) = \sigma_i u_i - \sigma_j u_j .
\]

\xhdr{Payoffs} Given a conformity profile $\lambda$, author $i$'s utility is 
\begin{align}
	\label{eq:utility}
	U_i(\lambda_i, \lambda_{-i})
	\;=&\;
	\underbrace{-\,\frac{b_i}{2}\,\big\| a_i(\lambda_i) - q^0 \big\|_2^2}_{\text{legibility reward}}
	\;\underbrace{-\,\frac{c_i}{2}\,\big\| a_i(\lambda_i) - r_i \big\|_2^2}_{\text{authenticity cost}}\notag\\
	\;&+\;
	\underbrace{\frac{\theta_i}{2(n-1)} \sum_{j \neq i} \big\| a_i(\lambda_i) - a_j(\lambda_j) \big\|_2^2}_{\text{distinctiveness payoff}} ,
\end{align}
where $b_i, \theta_i \ge 0$ and $c_i > 0$. The three terms represent, respectively, the benefit of proximity to the model-induced norm $q^0$, the private cost of deviating from the author's preferred style $r_i$ and the value of being distinguishable from other authors. The authenticity cost makes full conformity costly, while the distinctiveness payoff depends on the realized styles of others and couples with the author's choices.

Payoffs in \cref{eq:utility} use squared Euclidean distance rather than the Jensen--Shannon divergence used for $D^t$; JS is locally quadratic (\Cref{rem:quadratic}). We therefore evaluate long-run diversity using the quadratic measure introduced in \Cref{sec:convergence},
$\widehat{D}_\infty(\lambda):=\widehat{D}(a_1(\lambda_1),\ldots,a_n(\lambda_n))$,
which matches the payoff geometry and yields closed-form equilibria. We further make the following structural assumption.

\begin{definition}
	\label{def:orthogonal}
	A population has \emph{orthogonal signatures} if $\langle u_i, u_j \rangle = 0$ for all $i \neq j$, with $d_i^2 := \|u_i\|_2^2 > 0$~for~all~$i$.
\end{definition}

Orthogonality is a tractable benchmark yielding closed-form welfare results. Signatures lie in the $m-1$-dimensional tangent space of the simplex, so \Cref{def:orthogonal} requires $n \leq m-1$: the benchmark describes populations no larger than the feature dimension. This is not restrictive when $F$ is a rich inventory of lexical and syntactic markers, but it does couple population size to feature-set size. Appendix~\ref{app:correlated-signatures} relaxes orthogonality.

Under orthogonality,
$\| a_i - a_j \|_2^2 = \sigma_i^2 d_i^2 + \sigma_j^2 d_j^2$,
and substituting in \cref{eq:utility} gives the separable form
\begin{equation}
	\label{eq:utility-sep}
	U_i(\lambda)
	=
	\frac{d_i^2}{2}\Big[ (\theta_i\!-\!b_i)\, \sigma_i^2 \!-\! c_i\, (1-\sigma_i)^2 \Big]
	+\frac{\theta_i}{2(n-1)} \sum_{j \neq i} \sigma_j^2\, d_j^2 ,
\end{equation}
and the long-run diversity simplifies to
\begin{equation}
	\label{eq:diversity-sep}
	\widehat{D}_\infty(\lambda) = \frac{1}{n} \sum_{i=1}^{n} \sigma_i^2\, d_i^2 .
\end{equation}

The coupling in \cref{eq:utility-sep} is purely through the second term: author $i$'s choice does not change their own best response, but it does change everyone else's \emph{payoff}.

\begin{restatable}[]{lemma}{lemmaexternality} \label{lem:externality}
	Under orthogonal signatures, for every $i \neq j$,
	\[
	\frac{\partial U_j}{\partial \lambda_i}
	\;=\;
	-\,\frac{\theta_j}{n-1}\, \sigma_i\, d_i^2
	\;\le\; 0,
	\]
	with strict inequality whenever $\theta_j > 0$ and $\lambda_i < 1$. In particular, an increase in any author's conformity strictly reduces the payoff of every other author who places positive value on distinctiveness, including authors who choose not to conform.
\end{restatable}

Pairwise contrast is a shared resource: the distance $\|a_i - a_j\|^2$ enters both $U_i$ and $U_j$, so when author $i$ moves toward the norm they consume contrast that author $j$ was also drawing on. \Cref{lem:externality} also resolves the strategic question for non-adapters: an author for whom resisting is optimal still bears the cost of others' adaptation, because the pool of styles against which their distinctiveness is measured collapses toward $q^0$. There is no insulation in holding out.

\xhdr{Equilibrium  exceeds optimal conformity} Next we compare the conformity levels chosen by self-interested authors at Nash equilibrium with those chosen by a social planner that maximizes total welfare. The next theorem shows that authors over-conform relative to the social optimum, which leads to weakly lower long-run diversity and weakly lower welfare at equilibrium.

\begin{restatable}[]{theorem}{thmwedge} \label{thm:wedge}
	Consider the game with payoffs in \Cref{eq:utility} under orthogonal signatures, with $c_i > 0$ for all $i$ and let $\bar{\theta}_{-i} := \frac{1}{n-1}\sum_{j \neq i} \theta_j$ denote the average distinctiveness value of the other authors. Then:
	\squishlist
	\item[(i)] Every author has a strictly dominant strategy, so the game has a unique Nash equilibrium $\lambda^{\mathrm{NE}}$, given by
	\[
	\lambda_i^{\mathrm{NE}} \;=\; \frac{(b_i - \theta_i)_+}{\,(b_i - \theta_i)_+ + c_i\,},
	\quad \text{where } (x)_+ := \max\{0, x\}.
	\]
	\item[(ii)] Utilitarian welfare $W(\lambda) := \sum_{i=1}^n U_i(\lambda)$ is maximized at the unique profile $\lambda^{\mathrm{SO}}$ given by
	\[
	\lambda_i^{\mathrm{SO}} \;=\; \frac{\big(b_i - \theta_i - \bar{\theta}_{-i}\big)_+}{\,\big(b_i - \theta_i - \bar{\theta}_{-i}\big)_+ + c_i\,}.
	\]
	\item[(iii)] For every author, $\lambda_i^{\mathrm{NE}} \ge \lambda_i^{\mathrm{SO}}$, with strict inequality if and only if $\lambda_i^{\mathrm{NE}} > 0$ and $\bar{\theta}_{-i} > 0$. That is, every author who conforms at all over-conforms relative to the social optimum, and consequently $\widehat{D}_\infty(\lambda^{\mathrm{NE}}) \le \widehat{D}_\infty(\lambda^{\mathrm{SO}})$ and $W(\lambda^{\mathrm{NE}}) \le W(\lambda^{\mathrm{SO}})$.
	\squishend
\end{restatable}

Author $i$'s privately optimal conformity treats their distinctiveness as worth $\theta_i$, while the planner values it at $\theta_i + \bar{\theta}_{-i}$, because every other author also derives contrast from author $i$'s signature. Each author's over-conformity increases with how much the rest of the community values contrast with them. Because the wedge depends on the average $\bar{\theta}_{-i}$, it need not vanish in large populations. Notably, since the equilibrium is in dominant strategies, the inefficiency is driven purely by the payoff externality of \Cref{lem:externality}---exactly as in algorithmic monoculture~\citep{kleinberg2021algorithmic}, where individually optimal reliance on a shared algorithm reduces aggregate welfare without any agent erring.

\subsection{The price of monoculture} 
We quantify the diversity loss caused by strategic conformity via the price of monoculture (\PoM), which compares the long-run diversity of the socially optimal conformity profile with that of Nash equilibrium. We evaluate \PoM in a symmetric setting, where all authors have the same conformity reward, authenticity cost, value for distinctiveness and distance from the model norm. In this case, the \PoM has a closed-form expression and reveals three regimes: no inefficiency, pure deadweight conformity and interior over-conformity with an arbitrarily large diversity loss.\looseness=-1

\begin{definition}
	\label{def:pom}
	The \emph{price of monoculture} of an instance is the ratio of socially optimal to equilibrium long-run diversity,
	\[
	\PoM \;:=\; \frac{\widehat{D}_\infty(\lambda^{\mathrm{SO}})}{\widehat{D}_\infty(\lambda^{\mathrm{NE}})} \;=\; \frac{\sum_i (\sigma_i^{\mathrm{SO}})^2 d_i^2}{\sum_i (\sigma_i^{\mathrm{NE}})^2 d_i^2} \;\ge\; 1 .
	\]
\end{definition}

\begin{restatable}[]{corollary}{corsymmetric} \label{cor:symmetric}
	Suppose $b_i = b$, $c_i = c$, $\theta_i = \theta$, and $d_i^2 = d^2$ for all $i$. Then:
	\squishlist
	\item[(i)] If $b \le \theta$: $\lambda^{\mathrm{NE}} = \lambda^{\mathrm{SO}} = 0$ and $\PoM = 1$. Distinctiveness is at least as valuable as conformity privately, so no inefficiency arises.
	\item[(ii)] If $\theta < b \le 2\theta$: $\lambda^{\mathrm{NE}} = \frac{b-\theta}{b+c-\theta} > 0$ while $\lambda^{\mathrm{SO}} = 0$, and
	\[
	\PoM = \left( \frac{b + c - \theta}{c} \right)^{\!2} .
	\]
	Every author rationally conforms, yet the planner would have \emph{no one} conform: equilibrium conformity is pure deadweight.
	\item[(iii)] If $b > 2\theta$: both levels are interior, $\lambda^{\mathrm{NE}} = \frac{b-\theta}{b+c-\theta} > \lambda^{\mathrm{SO}} = \frac{b-2\theta}{b+c-2\theta}$, and
	\[
	\PoM = \left( \frac{b + c - \theta}{\,b + c - 2\theta\,} \right)^{\!2}
	= \left( 1 + \frac{\theta}{\,b + c - 2\theta\,} \right)^{\!2} ,
	\]
	which is increasing in $\theta$, equals $1$ at $\theta = 0$, and can diverge along sequences with $b+c-2\theta\downarrow0$. The per-author welfare loss at equilibrium is
	\[
	\frac{1}{n}\Big( W(\lambda^{\mathrm{SO}}) - W(\lambda^{\mathrm{NE}}) \Big)
	=
	\frac{d^2}{2} \cdot \frac{c^2 \theta^2}{(b + c - 2\theta)(b + c - \theta)^2},
	\]
	which is strictly positive whenever $\theta > 0$.
	\squishend
\end{restatable}

Three observations are worth highlighting. First, the welfare loss vanishes at $\theta=0$ and is quadratic to leading order near zero; when distinctiveness has no value, conformity is simply beneficial standardization. Thus, $\theta$ distinguishes benign standardization from harmful monoculture. Second, in regime~($ii$), conformity is individually rational but socially wasteful: every author conforms although the planner prefers none. Third, since $b>2\theta$ implies $b+c-2\theta>c$, every regime satisfies $\PoM\leq(1+\theta/c)^2$. The inefficiency is finite for each fixed instance but not uniformly bounded: it can diverge whenever $\theta/c\to\infty$, including as $c\downarrow0$ with $\theta$ fixed. Appendix~\ref{app:remarks} discusses externalities on readers (\Cref{rem:readers}) and strategic conformity under recursive model updates (\Cref{rem:endogenous-norm}).

\section{Quantitative Comparison and Illustrations} \label{sec:experiments}
We simulate $n=100$ authors over $m=10$ abstract linguistic features for $100$ independent runs of $T=200$ steps
(see Appendix~\ref{app:experiments} for details).
\Cref{fig:linguistic_diversity} reports the mean and standard deviation of population-level linguistic diversity across runs. Panel~(a) compares the three interaction mechanisms using
$\beta=0.25$ and, for IM~3, $\rho=2/3$ ($\gamma=0.5$ and $\delta=0.25$). Across all mechanisms, LLM assistance initially reduces diversity, after which $D^t$ stabilizes at a
mechanism-specific level. 
Under the baseline parameterization, recursive updates of the shared model (IM~2) produce a faster decline in $D^t$ and a lower limiting diversity than IM~1. By \Cref{prop:common_conformity}, under common conformity the mechanisms have identical limiting pairwise geometry, so any JS gap is purely positional. In our baseline, heterogeneous $\lambda_i$ additionally produces the quadratic gap isolated in Appendix~\ref{app:experiments}. 
Panel~(b) examines personalization in IM~3. Holding $\beta=0.25$ fixed, we vary $\rho$ and set $\gamma=(1-\beta)\rho$ and $\delta=(1-\beta)(1-\rho)$, thereby reallocating update
weight from population-level to author-specific feedback. Diversity at $T=200$ increases monotonically with $\rho$, from approximately
$0.057$ at $\rho=0$ to $0.171$ at $\rho=1$, showing that stronger
personalization can preserve population-level linguistic diversity. Additional details and experiments are reported in
Appendix~\ref{app:experiments}.

\begin{figure}
	\centering
	\includegraphics[width=0.75\linewidth]{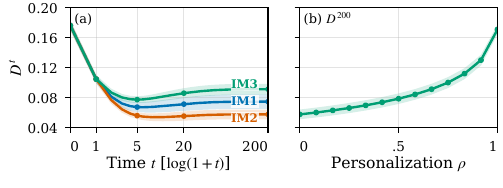}
	\caption{Population-level linguistic diversity under LLM assistance.
	Panel~(a) shows the evolution of $D^t$ under IM~1--3, with time displayed on a
	$\log(1+t)$ scale and tick labels reporting the original time steps.
	Panel~(b) shows IM~3 diversity at $T=200$ as personalization $\rho$ increases.
	Lines report means over $100$ runs and shading denotes $\pm1$ standard
	deviation; larger values indicate greater diversity.}
	\label{fig:linguistic_diversity}
\end{figure}

\section{Conclusions, Limitations and Future Work} \label{sec:conclusion}
We formalized linguistic monoculture as a population-level dynamical process over author and model distributions. Under our mechanisms, shared models can pull authors toward a common norm, while author-specific preferences and personalization can preserve diversity. Within our utility model, equilibrium conformity can exceed the social optimum, yielding a potentially unbounded price of monoculture in some regimes. When distinctiveness has little value, convergence can instead be efficient.

\xhdr{Limitations and future work} \label{app:limitations}
Our work introduces a mathematical framework for formalizing linguistic monoculture in LLM-assisted language use. The framework is stylized by design: it focuses on a limited set of interaction mechanisms that make it possible to analyze how linguistic diversity evolves under different forms of author--model interaction. As with many theoretical models, this tractability relies on simplifying assumptions, which allow us to obtain explicit convergence, equilibrium and welfare guarantees. These assumptions provide a foundation for identifying and formally analyzing mechanisms of linguistic convergence, while leaving richer behavioral and empirical refinements for future work. We therefore discuss limitations of the framework and outline directions for connecting it more closely to empirical settings and broadening its applicability.

\smallskip
\noindent
\textit{Heterogeneous non-LLM language exposure.}
Our framework isolates LLM-mediated adaptation from other influences, such as collaborators, reading research articles, reviewers, disciplinary norms, broader media environments and other human factors that shape language use. Thus, our results characterize convergence under controlled settings rather than a complete model of linguistic change under real-world conditions. Future work should incorporate heterogeneous non-LLM exposures, social networks, recursive norm-shaping incentives, and empirical calibration of conformity and personalization parameters from longitudinal corpora of LLM-assisted writing. 

\smallskip
\noindent
\textit{Fixed conformity, quadratic payoff and orthogonal signatures}
In \Cref{sec:strategic}, we make three simplifying assumptions to obtain closed-form expressions for equilibrium, welfare, and the price of monoculture. First, each author chooses a fixed conformity level $\lambda_i$, which determines their writing policy and remains unchanged over time. In practice, authors may update their conformity levels in response to experience, feedback, or changes in model behavior. Second, the analysis uses quadratic payoffs. Finally, the main-text closed-form results assume orthogonal author signatures, permitting at most $m-1$ nonzero signatures in $\Delta^{m-1}$. These assumptions isolate the conformity externality, but real signatures may be correlated and distinctiveness need not be valued quadratically. Extending the framework to dynamic conformity choices, richer similarity structures, and alternative payoff models is an important direction for future work.

\smallskip
\noindent
\textit{Empirical validation}
Our theoretical analysis is supported by simulations intended to illustrate the qualitative behavior predicted by our framework, rather than to provide empirically calibrated forecasts. The parameter choices allow us to compare interaction mechanisms under controlled conditions and should not be interpreted as estimates of linguistic convergence.
A natural next step is to verify these findings in controlled longitudinal studies of LLM-assisted writing, measuring how authors' linguistic-style distributions evolve under fixed, recursively updated and personalized assistance. In this sense, our framework provides a foundation for future empirical work by identifying the interaction mechanisms, parameters, and diversity measures that such studies could estimate or experimentally vary.

\medskip
\xhdr{Acknowledgments and Funding}
Suhas Thejaswi acknowledges support from the Technology Industries of Finland Centennial Foundation through a grant awarded to Aalto University. The authors declare that the funding source does not create any conflict of interest in relation to this work.


\medskip
\xhdr{GenAI usage statement} LLMs were used to edit and polish author-written text, and to assist with implementation and code refinement.

\bibliographystyle{plainnat}
\bibliography{refs}

\cleardoublepage
\appendix

\section{Remarks and Further Clarifications} 
\label{app:remarks}
This section provides additional clarifications and remarks that are omitted from the main text due to space constraints.

\begin{remark}[Distributional representation of linguistic style]
	\label{rem:distributional}
	We represent linguistic style as a distribution rather than a single point in feature space because neither authors nor models exhibit one fixed style across contexts. Human language varies systematically with task, genre, and audience---the linguistic notion of register~\citep{pescuma2023situating,biber2019register}---and the style a model induces depends on the prompt, task, and framing~\citep{brucks2025prompt}. A distribution captures this context-dependent variation while still giving a well-defined way to measure diversity across authors. The representation is also general: to study convergence in a specific feature (\eg, syntactic or formatting conventions) or register (\eg, academic prose), the feature set $\mathcal{F}$ can be restricted accordingly while other aspects of language are treated as fixed.
\end{remark}

\begin{remark}[Reparameterizing the recursive model update]
	\label{rem:reparam}
	The model update in \Cref{problem2} can equivalently be written as
	\[
	q^{t+1} = \beta q^t + \sum_{j=1}^n \gamma_j p_j^{t+1},\quad
	\text{where}\,\, \beta \geq 0,\ \gamma_j \geq 0,\ \beta+\sum_{j=1}^n \gamma_j = 1.
	\]
	Here, $\gamma_j$ is the contribution of author $j$'s output to the model update, and setting $\gamma_j=(1-\beta)w_j$ recovers the weighted-average aggregation. This form keeps the update a convex combination of the old model distribution and the new population data.
\end{remark}

\begin{remark}[Quadratic surrogate for JS]
	\label{rem:quadratic}
	Payoffs in \eqref{eq:utility} use squared Euclidean distance rather than the Jensen--Shannon divergence used for $D^t$. The two agree to second order: for nearby distributions,
	\[
	\mathrm{JS}(p,q)=\frac{1}{8}\sum_k\frac{(p_k-q_k)^2}{x_k}
	+o(\|p-q\|^2),\qquad x=\frac{p+q}{2}.
	\]
	Thus, when feature probabilities are bounded below, the two metrics are equivalent up to constants. Unlike squared Euclidean distance, however, JS also depends on the midpoint $x$: clouds with identical pairwise difference vectors can receive different JS diversity when located in different regions of the simplex. All results in Section~\ref{sec:strategic} concern $\widehat{D}_\infty$ and the quadratic payoffs in \eqref{eq:utility}; Appendix~\ref{app:experiments} additionally uses a translation-invariant quadratic diagnostic to separate changes in pairwise geometry from this positional sensitivity.
\end{remark}

\begin{remark}[Reader welfare]
	\label{rem:readers}
	\Cref{thm:wedge} counts only authors' payoffs in welfare, so each pairwise contrast is valued by exactly its two endpoints and the wedge $\bar{\theta}_{-i}$ is independent of $n$. If population diversity also benefits agents outside the game---readers, students, downstream researchers who never choose a $\lambda_i$ but consume the variety the community produces---then welfare contains an additional term $\Theta \cdot \widehat{D}_\infty(\lambda)$ that no author internalizes even partially. Under orthogonal signatures, this adds $\frac{\Theta}{n}\sigma_i^2 d_i^2$ to the social value of author $i$'s retained distinctiveness. Thus the planner's effective distinctiveness value for author $i$ increases by $2\Theta/n$ in the first-order condition, and if the societal value $\Theta$ scales with the size of the audience, the wedge grows with the community's reach. The inefficiency identified in the main text is therefore a lower bound: it is the loss that remains even when only the strategic participants are counted.
\end{remark}

\begin{remark}[Strategic conformity under recursive updates]
	\label{rem:endogenous-norm}
	We have endogenized $\lambda$ in the fixed-model setting of \Cref{problem1}. In the recursive setting of \Cref{problem2}, the equilibrium norm itself depends on the conformity profile,
	\[
	q^\ast(\lambda) = \frac{\sum_i w_i (1-\lambda_i) r_i}{1 - \sum_i w_i \lambda_i},
	\]
	by \Cref{prop:shared_recursive_conformity}. An author's conformity then exerts a second externality: by feeding less of their signature back into training data, they shift the shared norm itself toward the signatures of the remaining non-conformists, reweighting whose style defines ``polished'' for everyone. For $\lambda_i<1$, the fixed point remains well-defined and can be viewed as a weighted average of the preferred distributions $r_i$, with weights proportional to $w_i(1-\lambda_i)$. The expression becomes singular only at the boundary where all $\lambda_i=1$; as the profile approaches that boundary, the limiting norm can depend on the relative rates at which the terms $1-\lambda_i$ vanish. This boundary behavior is the strategic counterpart of the singularity noted after \Cref{prop:shared_recursive_conformity} and of model-collapse dynamics~\citep{shumailov2024ai}. Characterizing equilibria of the two-externality game, and whether norm-shaping incentives dampen or amplify over-conformity, is an open direction we view as the natural next step.
\end{remark}

\section{Omitted Proofs from Section~\ref{sec:convergence}}
\label{app:proofs-convergence}

\proplinguisticconvergence*
\begin{proof}
	Let $a:=\Acal(q^0,z)$. By assumption, the update rule becomes
	\[
	p_i^{t+1} = (1-\alpha_i)\,p_i^t + \alpha_i\,a.
	\]
	By subtracting $a$ on both sides followed by induction on $t$, we get,
	\begin{align*}
		p_i^{t+1} - a &= (1-\alpha_i)(\,p_i^t - a)\\
		p_i^{t} - a &= (1-\alpha_i)^t\,(p_i^0 - a)
	\end{align*}
	Taking $\ell_1$-norms and relying on the fact that both $p_i^{0}$ and $a$ are probability distributions, so their $\ell_1$ distance is at most $2$, we have
	\begin{align*}
		\|p_i^{t} - a\|_1 &= (1-\alpha_i)^t\,\|\left(p_i^0 - a\right)\|_1
		\leq 2\,(1-\alpha_i)^t.
	\end{align*}
	Using the standard inequality $1-x \leq e^{-x}$ for $x \in [0,1]$, $(1-\alpha_i)^t \leq e^{- \alpha_i t}$, we get
	\[
	\|p_i^{t} - a\|_1 \leq 2 e^{- \alpha_i t}.
	\]
	Now consider any pair of authors $i,j$. By triangle inequality and using the bound above, it follows that,
	\begin{align*}
		\|p_i^t-p_j^t\|_1 &\leq \|p_i^t- a \|_1+\|p_j^t - a\|_1 \\
		&\leq 2e^{-\alpha_{i}t} + 2e^{-\alpha_{j}t}
		\leq 4e^{-\alpha_{\min} \, t}.
	\end{align*}
	Now, using the standard bound $JS(p_i^t,p_j^t) \leq\frac{1}{2}\|p_i^t-p_j^t\|_1$, we obtain
	$JS(p_i^t,p_j^t) \leq 2e^{- \alpha_{\min} \, t}$.
	And, averaging over all ordered pairs $i \neq j$, we get
	\[
	D^t = \frac{1}{n(n-1)} \sum_{i\neq j}JS(p_i^t,p_j^t) \leq 2e^{-\alpha_{\min}t}.
	\]
	Therefore, to ensure $D^t \leq \varepsilon$, it is sufficient that
	\begin{align*}
		2e^{-\alpha_{\min}\, t} \leq \varepsilon \implies
		t \geq \frac{\log(2/ \varepsilon)}{\alpha_{\min}}
		= O \left(\frac{\log(1/\varepsilon)}{\alpha_{\min}}\right),
	\end{align*}
	which completes the proof.
\end{proof}

\propauthorspecificfixedadaptation*
\begin{proof}[Proof of \Cref{prop:author_specific_fixed_adaptation}]
	For each author $i$, define $a_i:=\Acal_i(q^0,z_i^t)$, which is fixed over time by assumption. The update rule becomes
	\[
	p_i^{t+1} = (1-\alpha_i)\,p_i^t + \alpha_i\, a_i.
	\]
	Subtracting \(a_i\) from both sides, followed by induction gives
	\begin{align*}
		p_i^{t+1}-a_i &= (1-\alpha_i)\,(p_i^t-a_i) \\
		p_i^t-a_i &= (1-\alpha_i)^t\,(p_i^0-a_i).
	\end{align*}
	Taking $\ell_1$-norms and using the fact that both $p_i^0$ and $a_i$ are probability distributions whose $\ell_1$ distance is at most $2$, gives
	\[
	\|p_i^t-a_i\|_1 \leq 2\,(1-\alpha_i)^t \leq 2\,e^{-\alpha_i t}
	\leq 2\,e^{-\alpha_{\min}t}.
	\]
	Thus $p_i^t\to a_i$ for every author $i$, as $t \to \infty$.
	Since Jensen--Shannon divergence is continuous on the probability simplex, for every pair $i,j$,
	\[
	JS(p_i^t,p_j^t) \to JS(a_i,a_j).
	\]
	Averaging over all ordered pairs $i\neq j$, we obtain
	\[
	D^t = \frac{1}{n(n-1)} \sum_{i\neq j}JS(p_i^t,p_j^t) \to \frac{1}{n(n-1)} \sum_{i\neq j}JS(a_i,a_j)
	=D^\infty.
	\]
	Therefore \(D^t\to 0\) if \(a_i=a_j\) for all \(i,j\). Conversely, since \(JS(a_i,a_j)=0\) if and only if \(a_i=a_j\), we have \(D^\infty=0\) only when all limiting adapted distributions are identical.
	
	\xhdr{Convergence of linguistic diversity}
	The time steps necessary for convergence of authors to their own limits within $\varepsilon >0$, \ie $\|p_i^t-a_i\|_1 \leq \varepsilon$, follows from the analysis in \Cref{prop:linguistic_convergence}, and it is enough to require $t \geq \frac{\log(2/\varepsilon)}{\alpha_{\min}}$. However, if $D^\infty>0$, then for any $\varepsilon < D^\infty$, there is no convergence time to $D^t\leq \varepsilon$. The process stabilizes at a positive level of linguistic diversity.
\end{proof}

\propfixedconformitymixture*
\begin{proof}[Proof of \Cref{prop:fixed_conformity_mixture}]
	Using
	$a_i(\lambda_i) := (1-\lambda_i)r_i+\lambda_i q^0$,
	the update rule of the author's distribution can be written as
	\begin{align*}
		p_i^{t+1} &= (1-\alpha_i)p_i^t+\alpha_i a_i(\lambda_i)\\
		p_i^{t+1}-a_i(\lambda_i) &= (1-\alpha_i)\left(p_i^t-a_i(\lambda_i)\right)\\
		p_i^t-a_i(\lambda_i) &= (1-\alpha_i)^t\left(p_i^0-a_i(\lambda_i)\right)\\
		\|p_i^t-a_i(\lambda_i)\|_1 &= (1-\alpha_i)^t\|p_i^0-a_i(\lambda_i)\|_1,
	\end{align*}
	where the second line subtracts $a_i(\lambda_i)$ on both sides, the third follows by induction, and the fourth takes $\ell_1$-norms.
	Since both $p_i^0$ and $a_i(\lambda_i)$ are probability distributions, their $\ell_1$-distance is at most $2$. Therefore, using the standard inequality $1-x\leq e^{-x}$ for $x\in[0,1]$ and $\alpha_{\min}:= \min_{i\in[n]} \alpha_i$,
	\begin{align*}
		\|p_i^t-a_i(\lambda_i)\|_1 &\leq 2(1-\alpha_i)^t
		\leq 2e^{-\alpha_i t}
		\leq 2e^{-\alpha_{\min}t},\\
		\max_{i\in[n]}\|p_i^t-a_i(\lambda_i)\|_1 &\leq 2e^{-\alpha_{\min}t}.
	\end{align*}
	To guarantee $\max_{i\in[n]}\|p_i^t-a_i(\lambda_i)\|_1\leq \varepsilon$, it is sufficient that
	$2e^{-\alpha_{\min}t}\leq \varepsilon$, equivalently
	$t\geq \frac{\log(2/\varepsilon)}{\alpha_{\min}}$.
	This proves the claimed convergence-time bound.
	
	Since \(p_i^t\to a_i(\lambda_i)\) for every author \(i\), and Jensen--Shannon divergence is continuous on the probability simplex,
	\[
	JS(p_i^t,p_j^t) \to JS(a_i(\lambda_i),a_j(\lambda_j)).
	\]
	Averaging over all ordered pairs \(i\neq j\) gives
	\[
	D^t \to D^\infty(\lambda) = \frac{1}{n(n-1)} \sum_{i\neq j} JS\left(a_i(\lambda_i),a_j(\lambda_j)\right).
	\]
	Because every Jensen--Shannon term is nonnegative and
	$JS(p,q)=0$ if and only if $p=q$, we have
	$D^\infty(\lambda)=0$ if and only if
	$a_i(\lambda_i)=a_j(\lambda_j)$ for all $i,j$.
	In particular, if $\lambda_i=1$ for all $i$, then
	$a_i(\lambda_i)=q^0$ for every $i$, and therefore
	$D^t\to 0$.
	
	Finally, suppose that $\lambda_i=\lambda$ for all authors.
	For any $0\leq\lambda_1<\lambda_2\leq 1$, define
	\[
	c:=\frac{1-\lambda_2}{1-\lambda_1}\in[0,1].
	\]
	Then, for every author $i$,
	\[
	a_i(\lambda_2)
	=
	c\,a_i(\lambda_1)+(1-c)q^0.
	\]
	By joint convexity of Jensen--Shannon divergence,
	\[
	JS\bigl(a_i(\lambda_2),a_j(\lambda_2)\bigr)
	\leq
	c\,JS\bigl(a_i(\lambda_1),a_j(\lambda_1)\bigr).
	\]
	Averaging over all ordered pairs shows that
	$D^\infty(\lambda)$ is nonincreasing in the common conformity
	level $\lambda$.
	
	Moreover,
	\[
	a_i(\lambda)-a_j(\lambda)
	=
	(1-\lambda)(r_i-r_j),
	\qquad
	\|a_i(\lambda)-a_j(\lambda)\|_1
	=
	(1-\lambda)\|r_i-r_j\|_1.
	\]
	Using $JS(p,q)\leq \frac12\|p-q\|_1$, we obtain
	\[
	D^\infty(\lambda)
	\leq
	(1-\lambda)
	\frac{1}{2n(n-1)}
	\sum_{i\neq j}\|r_i-r_j\|_1.
	\]
	Thus $D^\infty(\lambda)\to0$ as $\lambda\to1$, and
	$D^\infty(\lambda)=O(1-\lambda)$.
\end{proof}

\propsharedrecursiveconformity*
\begin{proof}[Proof of \Cref{prop:shared_recursive_conformity}]
	At equilibrium $p_i^{t+1} = p_i^t = p_i^\ast$ and $q^t=q^\ast$. Since $\alpha_i > 0$, the author distributions satisfy
	\begin{align} \label{eq:author_prob_equi}
		p_i^\ast &= (1-\alpha_i)p_i^\ast + \alpha_i\left((1-\lambda_i)r_i+\lambda_i q^\ast\right) \notag\\
		\alpha_i\,p_i^\ast &= \alpha_i\left((1-\lambda_i)r_i+\lambda_i q^\ast\right) \notag\\
		p_i^\ast &=(1-\lambda_i)r_i+\lambda_i q^\ast
	\end{align}
	Moreover, at equilibrium the model update satisfies
	\begin{align*}
		q^\ast &=\beta q^\ast+(1-\beta)P^\ast, \\
		q^\ast &= P^\ast = \sum_{i=1}^n w_i p_i^\ast \\
		q^\ast &= \sum_{i=1}^n w_i \bigl((1-\lambda_i)r_i+\lambda_i q^\ast\bigr)\\
		q^\ast &= \sum_{i=1}^n w_i(1-\lambda_i)r_i + \left(\sum_{i=1}^n w_i\lambda_i\right)q^\ast,
	\end{align*}
	where the third line substitutes \Cref{eq:author_prob_equi}.
	Since $\lambda_i<1$ for every $i$, and since $w_i\geq 0$ with $\sum_i w_i=1$, we have $\sum_{i=1}^{n}w_i \lambda_i<1$. Therefore,
	\begin{align} \label{eq:model_prob_equi}
		q^\ast &= \frac{\sum_{i=1}^n w_i(1-\lambda_i)\,r_i} {1-\sum_{i=1}^n w_i\lambda_i}.
	\end{align}
	Equations \eqref{eq:author_prob_equi} and \eqref{eq:model_prob_equi} give the claimed expressions for $p_i^\ast$ and $q^\ast$ at equilibrium.
	It remains to show that the distributions converge to this equilibrium.
	Define the deviations from equilibrium by $x_i^t:=p_i^t-p_i^\ast$ and $y^t:=q^t-q^\ast$. Using the update rule for $p_i^{t+1}$ and the equilibrium identity for $p_i^\ast$,
	\begin{align*}
		x_i^{t+1} &=p_i^{t+1}-p_i^\ast
		= (1-\alpha_i)(p_i^t-p_i^\ast) + \alpha_i\lambda_i(q^t-q^\ast)\\
		&= (1-\alpha_i)x_i^t+\alpha_i\lambda_i y^t, \\
		\|x_i^{t+1}\|_1 & \leq (1-\alpha_i)\,\|x_i^t\|_1 + \alpha_i\, \lambda_i \|y^t\|_1.
	\end{align*}
	Let $R^t:=\max \left\{\max_{i\in[n]}\|x_i^t\|_1,\, \|y^t\|_1 \right\}$ and $\mu:=\min_{i\in[n]} \alpha_i\,(1-\lambda_i)$.
	Since $\alpha_i>0$ and $\lambda_i<1$ for all $i$, we have $\mu>0$. Then, for every $i$,
	\[
	\|x_i^{t+1}\|_1 \leq (1-\alpha_i)R^t+\alpha_i\lambda_i R^t = \left(1-\alpha_i(1-\lambda_i)\right)R^t
	\leq (1-\mu)R^t.
	\]
	Next, consider the model deviation $y^{t+1}=q^{t+1}- q^\ast$. Using the update rule
	$q^{t+1}=\beta q^t+(1-\beta)P^{t+1}$
	and the equilibrium identity
	$q^\ast=\beta q^\ast+(1-\beta)P^\ast$,
	we get
	\[ y^{t+1} = \beta(q^t-q^\ast)+(1-\beta)(P^{t+1}-P^\ast).\]
	Using $P^{t+1}=\sum_{i=1}^n w_i p_i^{t+1}$ and $P^\ast=\sum_{i=1}^n w_i p_i^\ast$, we have
	$
	P^{t+1}-P^\ast = \sum_{i=1}^n w_i x_i^{t+1}.
	$
	Substituting back and using $\|x_i^{t+1}\|_1\leq(1-\mu)R^t$ together with $\sum_i w_i=1$,
	\begin{align*}
		y^{t+1} &= \beta y^t+(1-\beta)\sum_{i=1}^n w_i x_i^{t+1} \\
		\|y^{t+1}\|_1 &\leq \beta\|y^t\|_1 + (1-\beta)\sum_{i=1}^n w_i\|x_i^{t+1}\|_1 \\
		&\leq \bigl(\beta+(1-\beta)(1-\mu)\bigr)R^t
		= \bigl(1-(1-\beta)\mu\bigr)R^t.
	\end{align*}
	Let $\kappa:=1-(1-\beta)\mu$. Since $\beta<1$ and $\mu>0$, we have $\kappa<1$. Also $\kappa\geq 1-\mu$, because
	$\kappa=1-(1-\beta)\mu \geq 1-\mu$.
	Therefore, combining the bounds for $x_i^{t+1}$ and $y^{t+1}$, we get
	$R^{t+1}\leq \kappa R^t$,
	and applying this recursively yields $R^t\leq \kappa^t R^0$.
	Since $\kappa<1$, it follows that $R^t\to 0$ as $t\to\infty$. Hence,
	$p_i^t\to p_i^\ast$ for every $i$, and $q^t\to q^\ast$. This establishes the convergence.
	
	Since Jensen--Shannon divergence is continuous on the probability simplex,
	$JS(p_i^t,p_j^t)\to JS(p_i^\ast,p_j^\ast)$ for every pair $i,j$,
	and averaging over all ordered pairs $i\neq j$ yields
	\[
	D^t = \frac{1}{n(n-1)} \sum_{i\neq j}JS(p_i^t,p_j^t)
	\to \frac{1}{n(n-1)} \sum_{i\neq j}JS(p_i^\ast,p_j^\ast)
	= D^\ast.
	\]
	
	\xhdr{Convergence time}
	The convergence is exponential. Since all initial and equilibrium quantities are probability distributions, $R^0\le 2$, so
	\[
	\max\left\{\max_i\|p_i^t-p_i^\ast\|_1,\|q^t-q^\ast\|_1\right\}
	\leq
	2\kappa^t.
	\]
	Consequently, to guarantee
	$\max_i\|p_i^t-p_i^\ast\|_1\leq \varepsilon$ and $\|q^t-q^\ast\|_1\leq \varepsilon$,
	it is sufficient to take
	\[
	t \geq \frac{\log(2/\varepsilon)} {(1-\beta)\min_{i \in [n]} \alpha_i(1-\lambda_i)}.
	\]
	Thus the worst-case convergence guarantee for the recursive-update setting is weaker than for the fixed-model setting, where the corresponding bound is controlled only by \(\alpha_{\min}\); the bound need not be tight, and empirically the recursive dynamics can homogenize faster (\Cref{sec:experiments}).
	This completes the proof.
\end{proof}

\proppersonalizedrecursive*
\begin{proof}[Proof of \Cref{prop:personalized_recursive}]
	First, we characterize the equilibrium, and then show that the dynamics converge to it.
	
	\xhdr{Characterizing the equilibrium}
	At equilibrium, for every author $i$, we have
	$p_i^{t+1}=p_i^t=p_i^\ast$ and $q_i^{t+1}=q_i^t=q_i^\ast$.
	The author update gives
	\begin{align*}
		p_i^\ast &= (1-\alpha_i)p_i^\ast + \alpha_i\bigl((1-\lambda_i)r_i+\lambda_i q_i^\ast\bigr)\\
		p_i^\ast & =(1-\lambda_i)r_i+\lambda_i q_i^\ast && \text{since $\alpha_i>0$.}
	\end{align*}
	Next, the personalized model update gives, with $P^\ast:=\sum_{j=1}^n w_j p_j^\ast$,
	\begin{align*}
		q_i^\ast &= \beta q_i^\ast+\gamma p_i^\ast +\delta P^\ast \\
		(1-\beta)q_i^\ast &=\gamma p_i^\ast+\delta P^\ast \\
		q_i^\ast &=\frac{\gamma}{1-\beta} p_i^\ast+ \frac{\delta}{1-\beta} P^\ast.
	\end{align*}
	Using $\gamma+\delta=1-\beta$ and $\rho:=\frac{\gamma}{1-\beta}$, this becomes
	\[
	q_i^\ast =\rho p_i^\ast+(1-\rho)P^\ast.
	\]
	We now solve for \(p_i^\ast\). Substituting the expression for \(q_i^\ast\) into the equilibrium equation for \(p_i^\ast\),
	\begin{align*}
		p_i^\ast &= (1-\lambda_i)r_i+\lambda_i\bigl(\rho p_i^\ast+(1-\rho)P^\ast\bigr) \\
		(1-\rho\lambda_i)p_i^\ast &= (1-\lambda_i)r_i+\lambda_i(1-\rho)P^\ast \\
		p_i^\ast &= \frac{1-\lambda_i}{1-\rho\lambda_i}r_i +
		\frac{\lambda_i(1-\rho)}{1-\rho\lambda_i}P^\ast,
	\end{align*}
	where the division is valid since \(\lambda_i<1\) and \(\rho\in[0,1]\) imply \(1-\rho\lambda_i>0\).
	Defining
	$\eta_i:=\frac{1-\lambda_i}{1-\rho\lambda_i}$, we have
	$1-\eta_i = \frac{\lambda_i(1-\rho)}{1-\rho\lambda_i}$,
	and therefore
	\[
	p_i^\ast=\eta_i r_i+(1-\eta_i)P^\ast.
	\]
	It remains to identify \(P^\ast\). By definition,
	$P^\ast=\sum_{i=1}^n w_i p_i^\ast$. Substituting the expression for $p_i^\ast$ and using $\sum_i w_i=1$,
	\begin{align*}
		P^\ast &= \sum_{i=1}^n w_i\eta_i r_i + \left(1-\sum_{i=1}^n w_i\eta_i\right)P^\ast\\
		\left(\sum_{i=1}^n w_i\eta_i\right)P^\ast &= \sum_{i=1}^n w_i\eta_i r_i \\
		P^\ast &= \frac{\sum_{i=1}^n w_i\eta_i r_i} {\sum_{i=1}^n w_i\eta_i},
	\end{align*}
	where the division is valid since \(\lambda_i<1\) implies \(\eta_i>0\), hence \(\sum_i w_i\eta_i>0\).
	Finally, substituting \(p_i^\ast=\eta_i r_i+(1-\eta_i)P^\ast\) into \(q_i^\ast =\rho p_i^\ast+(1-\rho)P^\ast\),
	\begin{align*}
		q_i^\ast &= \rho\eta_i r_i+\rho(1-\eta_i)P^\ast+(1-\rho)P^\ast
		=\rho\eta_i r_i+(1-\rho\eta_i)P^\ast.
	\end{align*}
	This proves the claimed form of the equilibrium.
	
	\xhdr{Convergence to equilibrium}
	Define the deviations
	$x_i^t:=p_i^t-p_i^\ast$ and $y_i^t:=q_i^t-q_i^\ast$.
	Using the author update and the equilibrium identity,
	\begin{align*}
		x_i^{t+1} &=p_i^{t+1}-p_i^\ast
		= (1-\alpha_i)(p_i^t-p_i^\ast) + \alpha_i\lambda_i(q_i^t-q_i^\ast) \\
		&=(1-\alpha_i)x_i^t+\alpha_i\lambda_i y_i^t, \\
		\|x_i^{t+1}\|_1 &\leq (1-\alpha_i)\|x_i^t\|_1+\alpha_i\lambda_i\|y_i^t\|_1.
	\end{align*}
	Define $R^t:= \max\left\{ \max_i\|x_i^t\|_1,\, \max_i\|y_i^t\|_1 \right\}$ and
	$\mu:=\min_{i\in[n]}\alpha_i(1-\lambda_i)$.
	Since \(\alpha_i>0\) and \(\lambda_i<1\) for every \(i\), we have \(\mu>0\). Then
	\[
	\|x_i^{t+1}\|_1 \leq \bigl(1-\alpha_i(1-\lambda_i)\bigr)R^t \leq (1-\mu)R^t.
	\]
	Next, consider the model deviations. From the model update and the equilibrium identity,
	\begin{align*}
		y_i^{t+1} &= \beta y_i^t+\gamma x_i^{t+1} +\delta\sum_{j=1}^n w_jx_j^{t+1} \\
		\|y_i^{t+1}\|_1 &\leq \beta\|y_i^t\|_1 + \gamma\|x_i^{t+1}\|_1 + \delta\sum_{j=1}^n w_j\|x_j^{t+1}\|_1 \\
		&\leq \left(\beta+(\gamma+\delta)(1-\mu)\right)R^t \\
		&= \left(\beta+(1-\beta)(1-\mu)\right)R^t
		= \left(1-(1-\beta)\mu\right)R^t,
	\end{align*}
	using $\|x_j^{t+1}\|_1\leq(1-\mu)R^t$, $\sum_j w_j=1$, and \(\gamma+\delta=1-\beta\).
	Let $\kappa:=1-(1-\beta)\mu$. Since $\beta<1$ and $\mu>0$, we have $\kappa<1$; moreover
	$\kappa\geq 1-\mu$. Therefore, the bounds for both $x_i^{t+1}$ and $y_i^{t+1}$ imply
	$R^{t+1}\leq \kappa R^t$, and recursively $R^t\leq \kappa^t R^0$. Since $\kappa<1$, we have $R^t\to 0$, so
	$p_i^t\to p_i^\ast$ and $q_i^t\to q_i^\ast$ for every author $i$, establishing convergence.
	
	\xhdr{Convergence time}
	Since all $p_i^0,p_i^\ast,q_i^0,q_i^\ast$ are probability distributions, their $\ell_1$-distances are at most $2$. Hence $R^0\leq 2$ and $R^t\leq 2\kappa^t$. To guarantee
	$\max_i\|p_i^t-p_i^\ast\|_1\leq \varepsilon$ and
	$\max_i\|q_i^t-q_i^\ast\|_1\leq \varepsilon$,
	it is sufficient that $2\kappa^t\leq \varepsilon$, equivalently $t\geq \frac{\log(2/\varepsilon)}{-\log\kappa}$. Since
	$\kappa=1-(1-\beta)\mu$ and $-\log(1-z)\geq z$ for $z\in(0,1)$, it is sufficient to take
	\[
	t\geq \frac{\log(2/\varepsilon)} {(1-\beta)\min_i\alpha_i(1-\lambda_i)}.
	\]
	Finally, since Jensen--Shannon divergence is continuous on the probability simplex, for every pair $i,j$,
	\[
	JS(p_i^t,p_j^t)\to JS(p_i^\ast,p_j^\ast), \qquad
	JS(q_i^t,q_j^t)\to JS(q_i^\ast,q_j^\ast).
	\]
	Averaging over all ordered pairs \(i\neq j\), we obtain $D^t \to D^\ast$ and $Q^t \to Q^\ast$. This completes the proof.
\end{proof}

\propcommonconformity*
\begin{proof}
    For IM~1, \Cref{prop:fixed_conformity_mixture} gives
    $a_i(\lambda)=(1-\lambda)r_i+\lambda q^0$, hence
    $a_i(\lambda)-a_j(\lambda)=(1-\lambda)(r_i-r_j)$. Under
    \Cref{prop:shared_recursive_conformity},
    $p_i^\ast=(1-\lambda)r_i+\lambda q^\ast$, giving the same difference for IM~2.
    Under \Cref{prop:personalized_recursive}, $\eta_i=\eta$ for every $i$, hence
    $P^\ast=\sum_i w_i r_i$ and $p_i^\ast=\eta r_i+(1-\eta)P^\ast$, so
    $p_i^\ast-p_j^\ast=\eta(r_i-r_j)$. Since $\widehat{D}$ is homogeneous of degree
    two in pairwise differences and $\eta/(1-\lambda)=1/(1-\rho\lambda)$, the
    diversity relation follows.
\end{proof}


\section{Omitted Proofs from Section~\ref{sec:strategic}}
\label{app:proofs-strategic}
\lemmaexternality*
\begin{proof}
	From \eqref{eq:utility-sep}, $\lambda_i$ enters $U_j$ ($j \neq i$) only through the term $\frac{\theta_j}{2(n-1)} \sigma_i^2 d_i^2$. Differentiating with $\sigma_i = 1-\lambda_i$ gives $\partial U_j / \partial \lambda_i = -\frac{\theta_j}{n-1} \sigma_i d_i^2$.
\end{proof}

\thmwedge*
\begin{proof}
	\emph{(i)} By \eqref{eq:utility-sep}, $U_i$ is additively separable: the only term involving $\lambda_i$ is
	\[
	h_i(\sigma_i) := \frac{d_i^2}{2}\Big[ (\theta_i - b_i)\sigma_i^2 - c_i(1-\sigma_i)^2 \Big],
	\]
	which is independent of $\lambda_{-i}$. Hence the maximizer of $h_i$ over $\sigma_i \in [0,1]$ is a dominant strategy. We have
	\[
	h_i'(\sigma_i) = d_i^2 \Big[ (\theta_i - b_i)\sigma_i + c_i(1-\sigma_i) \Big]
	= d_i^2 \Big[ c_i - (b_i + c_i - \theta_i)\sigma_i \Big].
	\]
	\emph{Case $b_i > \theta_i$.} Then $b_i + c_i - \theta_i > c_i > 0$, $h_i$ is strictly concave, and the unconstrained maximizer $\sigma_i^\ast = c_i / (b_i + c_i - \theta_i) \in (0,1)$ is interior, giving $\lambda_i^{\mathrm{NE}} = 1 - \sigma_i^\ast = (b_i - \theta_i)/(b_i + c_i - \theta_i)$.
	\emph{Case $b_i \le \theta_i$.} Then $h_i'(\sigma_i) = d_i^2\big[ c_i(1-\sigma_i) + (\theta_i - b_i)\sigma_i \big] \ge 0$ on $[0,1]$, strictly positive on $[0,1)$, so $h_i$ is maximized at $\sigma_i = 1$, i.e.\ $\lambda_i^{\mathrm{NE}} = 0$. Both cases match the stated formula, and uniqueness of the maximizer in each case gives uniqueness of the equilibrium.
	
	\emph{(ii)} Summing \eqref{eq:utility-sep} over $i$ and collecting, for each $i$, the coefficient of $\sigma_i^2 d_i^2/2$---namely $(\theta_i - b_i)$ from $U_i$ and $\frac{\theta_j}{n-1}$ from each $U_j$ with $j \neq i$, the latter summing to $\bar{\theta}_{-i}$---we obtain
	\[
	W(\lambda) \;=\; \sum_{i=1}^{n} \frac{d_i^2}{2} \Big[ \big(\theta_i + \bar{\theta}_{-i} - b_i\big)\, \sigma_i^2 \;-\; c_i\,(1-\sigma_i)^2 \Big].
	\]
	This is again additively separable across authors, and each summand has exactly the form $h_i$ with $\theta_i$ replaced by $\theta_i + \bar{\theta}_{-i}$. The argument of part (i) applied verbatim with this replacement yields the stated $\lambda_i^{\mathrm{SO}}$ and its uniqueness.
	
	\emph{(iii)} Define, for $x \ge 0$,
	\[
	\varphi_i(x) := \frac{(b_i - \theta_i - x)_+}{\,(b_i - \theta_i - x)_+ + c_i\,},
	\]
	so that $\lambda_i^{\mathrm{NE}} = \varphi_i(0)$ and $\lambda_i^{\mathrm{SO}} = \varphi_i(\bar{\theta}_{-i})$. The map $y \mapsto y/(y+c_i)$ is strictly increasing on $[0,\infty)$ and $x \mapsto (b_i - \theta_i - x)_+$ is nonincreasing, strictly decreasing while positive; hence $\varphi_i$ is nonincreasing, strictly decreasing while positive. This gives $\lambda_i^{\mathrm{SO}} \le \lambda_i^{\mathrm{NE}}$ with strictness exactly when $\varphi_i(0) > 0$ and $\bar{\theta}_{-i} > 0$. The diversity comparison follows from \eqref{eq:diversity-sep}, since $\sigma_i^{\mathrm{NE}} \le \sigma_i^{\mathrm{SO}}$ for every $i$; the welfare comparison holds because $\lambda^{\mathrm{SO}}$ maximizes $W$.
\end{proof}

\corsymmetric*
\begin{proof}
	Parts (i) and (ii) and the equilibrium expressions in (iii) follow from \Cref{thm:wedge} with $\bar{\theta}_{-i} = \theta$; the $\PoM$ expressions follow from \eqref{eq:diversity-sep}, since in the symmetric case $\widehat{D}_\infty = \sigma^2 d^2$ and $\sigma^{\mathrm{NE}} = \frac{c}{b+c-\theta}$, while $\sigma^{\mathrm{SO}} = 1$ in regime (ii) and $\sigma^{\mathrm{SO}} = \frac{c}{b+c-2\theta}$ in regime (iii).
	
	For the welfare loss in (iii), write $A := b + c - 2\theta > 0$. From the proof of \Cref{thm:wedge}(ii), per-author welfare at a symmetric profile $\sigma$ is $\frac{d^2}{2}\, w(\sigma)$ with
	\[
	w(\sigma) = (2\theta - b)\sigma^2 - c(1-\sigma)^2 = -A\sigma^2 + 2c\sigma - c .
	\]
	Then $w(\sigma^{\mathrm{SO}}) = w(c/A) = c^2/A - c$, while with $\sigma^{\mathrm{NE}} = c/(A+\theta)$,
	\[
	w(\sigma^{\mathrm{NE}}) = -\frac{A c^2}{(A+\theta)^2} + \frac{2c^2}{A+\theta} - c = \frac{c^2 (A + 2\theta)}{(A+\theta)^2} - c .
	\]
	Subtracting,
	\[
	w(\sigma^{\mathrm{SO}}) - w(\sigma^{\mathrm{NE}})
	= c^2 \cdot \frac{(A+\theta)^2 - A(A + 2\theta)}{A (A+\theta)^2}
	= \frac{c^2\, \theta^2}{A (A+\theta)^2},
	\]
	since $(A+\theta)^2 - A(A+2\theta) = \theta^2$. Substituting $A = b+c-2\theta$ and $A + \theta = b+c-\theta$ completes the proof.
\end{proof}


\section{Correlated Author Signatures}
\label{app:correlated-signatures}

We now relax the orthogonality assumption in Definition~1 and
characterize the strategic-conformity game for arbitrary author
signatures. Recall that
\[
u_i := r_i-q^0,
\qquad
d_i^2 := \lVert u_i\rVert_2^2,
\qquad
\sigma_i := 1-\lambda_i,
\]
so that the limiting style of author $i$ satisfies
\[
a_i(\lambda_i)-q^0=\sigma_i u_i.
\]
Let
\[
g_{ij}:=\langle u_i,u_j\rangle
\]
and let $G=(g_{ij})_{i,j\in[n]}$ denote the Gram matrix of the
author signatures. Thus, $g_{ii}=d_i^2$, while $g_{ij}$ measures
the alignment between the directions in which authors $i$ and $j$
differ from the shared model-induced norm. Orthogonal signatures
correspond to $g_{ij}=0$ for every $i\neq j$.

\begin{proposition}[Strategic conformity with correlated signatures]
\label{prop:correlated-signatures}
Consider the game with payoffs in Equation~(3), without imposing
orthogonality.

\begin{enumerate}
    \item For every author $i$,
    \begin{align}
    U_i(\sigma)
    ={}&
    \frac{d_i^2}{2}
    \left[
        (\theta_i-b_i)\sigma_i^2
        -c_i(1-\sigma_i)^2
    \right]
    \nonumber\\
    &+
    \frac{\theta_i}{2(n-1)}
    \sum_{j\neq i}d_j^2\sigma_j^2
    -
    \frac{\theta_i}{n-1}
    \sigma_i\sum_{j\neq i}g_{ij}\sigma_j.
    \label{eq:correlated-utility}
    \end{align}
    The corresponding quadratic long-run diversity is
    \begin{equation}
    D^\infty(\lambda)
    =
    \frac{1}{n}\sum_{i=1}^n d_i^2\sigma_i^2
    -
    \frac{2}{n(n-1)}
    \sum_{i<j}g_{ij}\sigma_i\sigma_j.
    \label{eq:correlated-diversity}
    \end{equation}

    \item For every pair $i\neq j$,
    \begin{equation}
    \frac{\partial U_j}{\partial\lambda_i}
    =
    \frac{\theta_j}{n-1}
    \left(
        \sigma_jg_{ij}-\sigma_i d_i^2
    \right).
    \label{eq:correlated-externality}
    \end{equation}
    Consequently, an increase in author $i$'s conformity imposes
    a nonpositive externality on author $j$ if and only if
    \begin{equation}
    \sigma_i d_i^2\geq \sigma_jg_{ij}.
    \label{eq:negative-externality-condition}
    \end{equation}
    In particular, the externality is nonpositive at every
    conformity profile whenever $g_{ij}\leq 0$.

    \item Define
    \[
    A_i:=b_i+c_i-\theta_i.
    \]
    If $A_i>0$, author $i$'s payoff is strictly concave in
    $\sigma_i$, conditional on $\sigma_{-i}$, and their unique
    best response is
    \begin{equation}
    \operatorname{BR}_i(\sigma_{-i})
    =
    \Pi_{[0,1]}
    \left[
        \frac{
            c_i-
            \dfrac{\theta_i}{(n-1)d_i^2}
            \sum_{j\neq i}g_{ij}\sigma_j
        }{
            A_i
        }
    \right],
    \label{eq:correlated-best-response}
    \end{equation}
    where $\Pi_{[0,1]}$ denotes projection onto $[0,1]$.

    If, in addition,
    \begin{equation}
    \max_{i\in[n]}
    \frac{\theta_i}
    {(n-1)d_i^2A_i}
    \sum_{j\neq i}|g_{ij}|
    <1,
    \label{eq:best-response-contraction}
    \end{equation}
    then the joint best-response map is a contraction and the game
    has a unique Nash equilibrium.

    If the Nash equilibrium is interior, its
    retained-distinctiveness vector satisfies
    \begin{equation}
    N\sigma^{\mathrm{NE}}=h,
    \label{eq:correlated-ne-system}
    \end{equation}
    where
    \[
    h_i:=c_i d_i^2
    \]
    and
    \begin{equation}
    N_{ii}=d_i^2(b_i+c_i-\theta_i),
    \qquad
    N_{ij}=\frac{\theta_i}{n-1}g_{ij},
    \quad i\neq j.
    \label{eq:correlated-ne-matrix}
    \end{equation}

    \item Let
    \[
    \bar{\theta}_{-i}
    :=
    \frac{1}{n-1}\sum_{j\neq i}\theta_j.
    \]
    Utilitarian welfare can be written as
    \begin{equation}
    W(\sigma)
    =
    C+h^\top\sigma-\frac{1}{2}\sigma^\top S\sigma,
    \label{eq:correlated-welfare}
    \end{equation}
    where $C$ is independent of $\sigma$, $h_i=c_i d_i^2$, and
    \begin{equation}
    S_{ii}
    =
    d_i^2
    \left(
        b_i+c_i-\theta_i-\bar{\theta}_{-i}
    \right),
    \label{eq:correlated-welfare-diagonal}
    \end{equation}
    \begin{equation}
    S_{ij}
    =
    \frac{\theta_i+\theta_j}{n-1}g_{ij},
    \qquad i\neq j.
    \label{eq:correlated-welfare-offdiagonal}
    \end{equation}
    If $S$ is positive definite, welfare is strictly concave and
    has a unique maximizer on $[0,1]^n$. If the maximizer is
    interior, it is characterized by
    \begin{equation}
    S\sigma^{\mathrm{SO}}=h.
    \label{eq:correlated-so-system}
    \end{equation}
\end{enumerate}
\end{proposition}

\begin{proof}
For any $i\neq j$,
\begin{align}
\lVert a_i-a_j\rVert_2^2
&=
\lVert \sigma_i u_i-\sigma_j u_j\rVert_2^2
\nonumber\\
&=
\sigma_i^2d_i^2+\sigma_j^2d_j^2
-2\sigma_i\sigma_jg_{ij}.
\label{eq:correlated-distance}
\end{align}
Substituting Equation~\eqref{eq:correlated-distance} into
Equation~(3), the terms involving $\sigma_i^2d_i^2$ appear once
for every $j\neq i$, giving
\[
\frac{\theta_i}{2(n-1)}
\sum_{j\neq i}\sigma_i^2d_i^2
=
\frac{\theta_i}{2}\sigma_i^2d_i^2.
\]
Collecting this term with the legibility and authenticity terms
gives the first line of Equation~\eqref{eq:correlated-utility}.
The remaining squared-signature and cross terms give the second
line.

For quadratic diversity, summing
Equation~\eqref{eq:correlated-distance} over unordered pairs gives
\[
\sum_{i<j}
\left(
    \sigma_i^2d_i^2+\sigma_j^2d_j^2
\right)
=
(n-1)\sum_i\sigma_i^2d_i^2.
\]
Substitution into the definition of quadratic long-run diversity
gives Equation~\eqref{eq:correlated-diversity}.

To obtain the externality formula, consider $i\neq j$. The only
part of $U_j$ that depends on $\lambda_i$ is the pairwise term
involving authors $i$ and $j$. Since
\[
\frac{\partial\sigma_i}{\partial\lambda_i}=-1,
\]
we have
\begin{align}
\frac{\partial U_j}{\partial\lambda_i}
&=
\frac{\theta_j}{2(n-1)}
\frac{\partial}{\partial\lambda_i}
\lVert\sigma_j u_j-\sigma_i u_i\rVert_2^2
\nonumber\\
&=
\frac{\theta_j}{n-1}
\left\langle
    \sigma_j u_j-\sigma_i u_i,u_i
\right\rangle
\nonumber\\
&=
\frac{\theta_j}{n-1}
\left(
    \sigma_jg_{ij}-\sigma_i d_i^2
\right).
\end{align}
This proves Equations~\eqref{eq:correlated-externality}
and~\eqref{eq:negative-externality-condition}. If $g_{ij}\leq0$,
both $\sigma_jg_{ij}$ and $-\sigma_i d_i^2$ are nonpositive, so
the externality is nonpositive.

Differentiating Equation~\eqref{eq:correlated-utility} with respect
to $\sigma_i$ gives
\begin{equation}
\frac{\partial U_i}{\partial\sigma_i}
=
d_i^2
\left[
    c_i-(b_i+c_i-\theta_i)\sigma_i
\right]
-
\frac{\theta_i}{n-1}
\sum_{j\neq i}g_{ij}\sigma_j.
\label{eq:correlated-own-derivative}
\end{equation}
Moreover,
\[
\frac{\partial^2U_i}{\partial\sigma_i^2}
=
-d_i^2A_i.
\]
Thus, $A_i>0$ implies strict concavity in $\sigma_i$. Solving
Equation~\eqref{eq:correlated-own-derivative} and projecting the
resulting unconstrained maximizer onto the feasible interval gives
Equation~\eqref{eq:correlated-best-response}.

Projection onto a closed interval is nonexpansive. Hence, for any
two profiles $\sigma$ and $\sigma'$,
\begin{align}
\left|
    \operatorname{BR}_i(\sigma_{-i})
    -
    \operatorname{BR}_i(\sigma'_{-i})
\right|
&\leq
\frac{\theta_i}
{(n-1)d_i^2A_i}
\sum_{j\neq i}
|g_{ij}|
\left|\sigma_j-\sigma'_j\right|
\nonumber\\
&\leq
\frac{\theta_i}
{(n-1)d_i^2A_i}
\sum_{j\neq i}
|g_{ij}|
\lVert\sigma-\sigma'\rVert_\infty.
\end{align}
Condition~\eqref{eq:best-response-contraction} therefore makes the
joint best-response map a contraction in the sup norm. Banach's
fixed-point theorem gives existence and uniqueness of the Nash
equilibrium. At an interior equilibrium,
Equation~\eqref{eq:correlated-own-derivative} is zero for every
$i$, which is exactly the linear system
$N\sigma^{\mathrm{NE}}=h$.

For welfare, each unordered pair $\{i,j\}$ appears in both
authors' distinctiveness payoffs, with total coefficient
\[
\frac{\theta_i+\theta_j}{2(n-1)}.
\]
Therefore,
\begin{align}
W(\sigma)
={}&
-\frac{1}{2}\sum_i
\left[
    b_i d_i^2\sigma_i^2
    +
    c_i d_i^2(1-\sigma_i)^2
\right]
\nonumber\\
&+
\sum_{i<j}
\frac{\theta_i+\theta_j}{2(n-1)}
\lVert\sigma_i u_i-\sigma_j u_j\rVert_2^2.
\label{eq:correlated-expanded-welfare}
\end{align}
Expanding the squared distances and collecting coefficients gives
Equations~\eqref{eq:correlated-welfare}--%
\eqref{eq:correlated-welfare-offdiagonal}. If $S$ is positive
definite, the Hessian of welfare is $-S$, so welfare is strictly
concave. Its interior first-order condition is
\[
\nabla W(\sigma)=h-S\sigma=0,
\]
which gives Equation~\eqref{eq:correlated-so-system}.
\end{proof}

\begin{corollary}[Over-conformity with nonpositively aligned signatures]
\label{cor:nonpositive-signatures}
Suppose
\begin{equation}
g_{ij}\leq0
\qquad
\text{for every }i\neq j.
\label{eq:nonpositive-alignments}
\end{equation}
Assume that the Nash equilibrium and social optimum are interior
and that
\begin{equation}
d_i^2(b_i+c_i-\theta_i)
>
\frac{\theta_i}{n-1}
\sum_{j\neq i}|g_{ij}|
\label{eq:ne-diagonal-dominance}
\end{equation}
and
\begin{equation}
d_i^2
\left(
    b_i+c_i-\theta_i-\bar{\theta}_{-i}
\right)
>
\frac{1}{n-1}
\sum_{j\neq i}
(\theta_i+\theta_j)|g_{ij}|
\label{eq:so-diagonal-dominance}
\end{equation}
for every author $i$. Then the Nash equilibrium and social
optimum are unique and
\begin{equation}
\sigma_i^{\mathrm{SO}}
\geq
\sigma_i^{\mathrm{NE}}
\qquad
\text{for every }i.
\label{eq:sigma-overconformity}
\end{equation}
Equivalently,
\begin{equation}
\lambda_i^{\mathrm{NE}}
\geq
\lambda_i^{\mathrm{SO}}
\qquad
\text{for every }i.
\label{eq:lambda-overconformity}
\end{equation}
The inequality for author $i$ is strict whenever
\begin{equation}
\sigma_i^{\mathrm{NE}}>0
\qquad\text{and}\qquad
\bar{\theta}_{-i}>0.
\label{eq:strict-overconformity-condition}
\end{equation}
Consequently,
\begin{equation}
D^\infty(\lambda^{\mathrm{NE}})
\leq
D^\infty(\lambda^{\mathrm{SO}}).
\label{eq:correlated-diversity-comparison}
\end{equation}
\end{corollary}

\begin{proof}
Under Equation~\eqref{eq:nonpositive-alignments}, both $N$ and
$S$ have nonpositive off-diagonal entries.
Conditions~\eqref{eq:ne-diagonal-dominance}
and~\eqref{eq:so-diagonal-dominance} give positive, strictly
dominant diagonal entries. Consequently, $N$ and $S$ are
nonsingular $M$-matrices and have entrywise nonnegative inverses.

Furthermore, $S\leq N$ entrywise. On the diagonal,
\[
S_{ii}
=
N_{ii}-d_i^2\bar{\theta}_{-i}
\leq N_{ii}.
\]
For $i\neq j$,
\[
S_{ij}-N_{ij}
=
\frac{\theta_j}{n-1}g_{ij}
\leq0.
\]
Because $\sigma^{\mathrm{NE}}\geq0$,
\[
S\sigma^{\mathrm{NE}}
\leq
N\sigma^{\mathrm{NE}}
=
h.
\]
Multiplication by the nonnegative matrix $S^{-1}$ gives
\[
\sigma^{\mathrm{NE}}
\leq
S^{-1}h
=
\sigma^{\mathrm{SO}},
\]
which proves Equations~\eqref{eq:sigma-overconformity}
and~\eqref{eq:lambda-overconformity}.

For strictness, observe that
\[
h-S\sigma^{\mathrm{NE}}
=
(N-S)\sigma^{\mathrm{NE}}.
\]
Under Equation~\eqref{eq:nonpositive-alignments}, the matrix
$N-S$ is entrywise nonnegative. Its $i$th diagonal contribution is
\[
d_i^2\bar{\theta}_{-i}\sigma_i^{\mathrm{NE}},
\]
which is strictly positive whenever
Equation~\eqref{eq:strict-overconformity-condition} holds.
Because $S^{-1}$ is entrywise nonnegative and has strictly
positive diagonal entries, it follows that
\[
\sigma_i^{\mathrm{SO}}>
\sigma_i^{\mathrm{NE}},
\]
or equivalently,
\[
\lambda_i^{\mathrm{NE}}>
\lambda_i^{\mathrm{SO}}.
\]

Finally, when $g_{ij}\leq0$, each pairwise squared distance
\[
\lVert\sigma_i u_i-\sigma_j u_j\rVert_2^2
=
\sigma_i^2d_i^2+\sigma_j^2d_j^2
-2\sigma_i\sigma_jg_{ij}
\]
is nondecreasing in each of $\sigma_i$ and $\sigma_j$ on
$[0,1]^2$. The componentwise inequality
$\sigma^{\mathrm{SO}}\geq\sigma^{\mathrm{NE}}$ therefore implies
Equation~\eqref{eq:correlated-diversity-comparison}.
\end{proof}

\begin{remark}[Positive alignment can reverse the externality]
\label{rem:positive-alignment}
For unrestricted correlated signatures, the conformity externality
need not be negative. For example, suppose that authors $i$ and
$j$ have identical signature vectors,
\[
u_i=u_j,
\]
so that
\[
g_{ij}=d_i^2=d_j^2.
\]
If $\sigma_j>\sigma_i$, Equation~\eqref{eq:correlated-externality}
gives
\[
\frac{\partial U_j}{\partial\lambda_i}
=
\frac{\theta_jd_i^2}{n-1}
(\sigma_j-\sigma_i)
>0.
\]
In this case, author $i$ already lies closer to the shared norm
than author $j$. Additional conformity by author $i$ moves their
realized styles farther apart and therefore increases author $j$'s
distinctiveness payoff. Thus, no theorem asserting a negative
conformity externality at every profile can hold for arbitrary
positively correlated signatures.

The orthogonal-signature result is therefore one member of a wider
class rather than an implication that holds for every signature
geometry. Orthogonality eliminates strategic cross terms and
yields the closed-form dominant-strategy expressions in
Theorem~4.2. More generally, nonpositive signature alignments
preserve the negative externality, and under the conditions of
Corollary~\ref{cor:nonpositive-signatures} they preserve
componentwise over-conformity. Positive alignments make the sign
and magnitude of the externality depend on the authors' relative
retained distinctiveness.
\end{remark}

\section{A Heterogeneous Population with Multiple Interaction \\Mechanisms}
\label{app:im4}
This section introduces a fourth interaction mechanism that combines Mechanisms 1, 2 and 3 within a single population by partitioning authors into subpopulations, each following the update rule of its assigned mechanism.

\begin{mechanism}[A heterogeneous population with multiple interaction mechanisms]
	\label{problem4}
	The population is partitioned into three disjoint subpopulations
	$
	[n] = \mathcal{I}_1 \cup \mathcal{I}_2 \cup \mathcal{I}_3,
	$
	where authors in $\Ical_1$, $\Ical_2$ and $\Ical_3$ follow Interaction Mechanisms 1, 2 and 3, respectively. Let $q^0$ denote the common base model distribution. 
	Authors in $\Ical_1$ interact with a shared model whose distribution remains fixed at $q_0$. Their linguistic-style distributions evolve as
	\[
	p_i^{t+1} = (1-\alpha_i)\,p_i^t+\alpha_i\, \Acal_i(q^0,r_i), \qquad i \in \Ical_1.
	\]
	Authors in $\Ical_2$ interact with a shared model whose linguistic style distribution ${q}^t$ is recursively updated using feedback from authors in $\Ical_2$. Their coupled author--model dynamics evolve as
	\begin{align*}
		p_i^{t+1} &= (1-\alpha_i)\, p_i^t + \alpha_i\, \Acal_i({q}^t,r_i), \qquad i\in \Ical_2,\\
		q^{t+1} &= \beta\, q^t+(1-\beta)\, \Bcal \big((p_j^{t+1})_{j\in\mathcal{I}_2}\big),
	\end{align*}
	where $\Bcal$ aggregates the updated author distributions in $\Ical_2$. In particular, we use the weighted-average update 
	\[
	\Bcal \big((p_j^{t+1})_{j\in\Ical_2}\big) = \sum_{j\in \Ical_2} w_j p_j^{t+1},
	\qquad w_j \ge 0,\quad \sum_ {j\in \Ical_2} w_j=1.
	\]
	Each author $i \in \Ical_3$ interacts with an author-specific personalized model $q_i^t$, initialized from a common base distribution $q_i^0=q^0$. Their coupled author--model dynamics are
	\begin{align*}
		p_i^{t+1} &= (1-\alpha_i)\, p_i^t + \alpha_i \Acal_i(q_i^t,r_i), \qquad i\in \Ical_3 ,\\
		q_i^{t+1} &= \beta\, q_i^t + \gamma\, p_i^{t+1} + \delta\, \sum_{j\in \Ical_3} w_j p_j^{t+1},
		\qquad i\in \Ical_3 ,
	\end{align*}
	where $\beta,\gamma,\delta\geq 0$, $\beta+\gamma+\delta=1$,
$w_j\geq 0$, and $\sum_{j\in\Ical_3} w_j=1$.
	The three subpopulations evolve according to their respective update rules, with no additional coupling across $\Ical_1$, $\Ical_2$ and $\Ical_3$. Thus, \Cref{problem4} captures a heterogeneous population in which different groups use LLM assistance through different interaction mechanisms.
\end{mechanism}

\begin{remark}[Convergence under Interaction Mechanism~4]
	Under the conformity-mixture adaptation operators and parameter conditions of \Cref{prop:fixed_conformity_mixture,prop:shared_recursive_conformity,prop:personalized_recursive}, Interaction Mechanism~4 does not introduce additional coupling across the three subpopulations. Therefore, convergence follows by applying the corresponding result separately within each subpopulation. Authors in $\Ical_1$ converge to the fixed-model limits characterized in \Cref{prop:fixed_conformity_mixture}; authors in $\Ical_2$, together with their shared recursively updated model, converge to the endogenous shared equilibrium characterized in \Cref{prop:shared_recursive_conformity}; and authors in $\Ical_3$, together with their personalized models, converge to the family of author-specific equilibria characterized in \Cref{prop:personalized_recursive}.
	
	Consequently, the population-level diversity $D^t$ also converges. Its limit is the average pairwise Jensen--Shannon divergence among the limiting author distributions across the whole population. In particular, if $p_i^\ast$ denotes the limiting distribution of author $i$ under the mechanism assigned to their subpopulation, then
	\[
	D^t\;\longrightarrow\; D^\ast_{\mathrm{mix}} :=\frac{1}{n(n-1)} \sum_{i\neq j} JS(p_i^\ast,p_j^\ast).
	\]
	Equivalently, this limiting diversity decomposes into within-subpopulation terms, determined by the limiting diversity inside each of IMs~1-3, and between-subpopulation terms, determined by the distances between the equilibria reached by authors using different mechanisms. Thus, the mixed mechanism can yield an intermediate level of diversity in simulations, because the population contains authors governed by fixed, recursive, and personalized interaction rules. However, the exact limiting value depends on the sizes of the subpopulations, their initial and preferred distributions, and the parameters of each mechanism; no universal ordering relative to IMs~1-3 follows without additional assumptions.
\end{remark}

\section{Additional Experimental Results}
\label{app:experiments}
This appendix reports robustness and diagnostic experiments complementing
Figure~\ref{fig:linguistic_diversity}. We study sensitivity to population size,
Dirichlet concentration, conformity heterogeneity, and the diversity metric,
together with auxiliary convergence measures, personalization dynamics, and a
heterogeneous population whose subgroups follow IM~1--3.
Our implementation\footnote{LLMs (\eg, ChatGPT Codex, CoPilot) were used to assist with implementation and code refinement. All generated code was reviewed and verified by the authors, who remain responsible for its correctness.} is written in \texttt{Python} programming language using standard numerical libraries. All experiments were executed on a commodity laptop. The source code is available as open-source under a modest license agreement.\footnote{\url{https://github.com/suhastheju/llm-monoculture-code}}

Unless stated otherwise, we simulate $n=100$ authors over $m=10$ abstract
linguistic features for $T=200$ steps. We initialize $p_i^0$, $q^0$, and $r_i$
independently from $\operatorname{Dirichlet}(1,\ldots,1)$ and use the same
initialization across mechanisms within each run. We independently sample
$\alpha_i\sim\operatorname{Uniform}[0.05,0.50)$ and
$\lambda_i\sim\operatorname{Uniform}[0,1)$, and use uniform author weights
$w_i=1/n$. The baseline parameters are $\beta=0.25$ and
$\rho=2/3$ (equivalently, $\gamma=0.5$ and $\delta=0.25$).
All panels report means over 100 independent paired runs, with shading
denoting $\pm1$ sample standard deviation. Time is displayed on a
$\log(1+t)$ scale to resolve the rapid initial transient; tick labels report
the original time steps. The master random seed is $123456789$.
The absolute plateau levels and the ordering between IM~1 and IM~2 depend on
the initialization and parameter prior. The experiments therefore illustrate
mechanism-specific dynamics in controlled parameter regimes rather than
establish a universal ordering.
Alongside the Jensen--Shannon diversity $D^t$, we report the translation-invariant
quadratic diagnostic
\[
\widetilde D^t
:=
\frac{m}{8n(n-1)}
\sum_{i\neq j}\|p_i^t-p_j^t\|_2^2.
\]
This quantity depends only on the pairwise difference vectors $p_i^t-p_j^t$,
so a common translation of an author cloud leaves it unchanged. The factor
$m/8$ matches the second-order JS scaling when pairwise midpoints lie at the
barycenter. We use $\widetilde D^t$ only as a diagnostic; the paper's primary
diversity measure remains $D^t$.

\paragraph{Dirichlet concentration and diversity metric.}
We first vary the initialization concentration
$a\in\{0.1,1,5\}$ in
$\operatorname{Dirichlet}(a,\ldots,a)$. The ordering
IM~2 $<$ IM~1 $<$ IM~3 holds at $T=200$ under both $D^t$ and
$\widetilde D^t$ for every value of $a$. For the baseline $a=1$, the quadratic
endpoints are approximately $0.067$ for IM~2, $0.083$ for IM~1, and $0.107$
for IM~3. Thus, the observed endpoint ordering is also present under the
quadratic diagnostic and is not solely an artifact of the positional
sensitivity of Jensen--Shannon divergence.

\begin{figure}[t]
    \centering
    \includegraphics[width=0.95\linewidth]{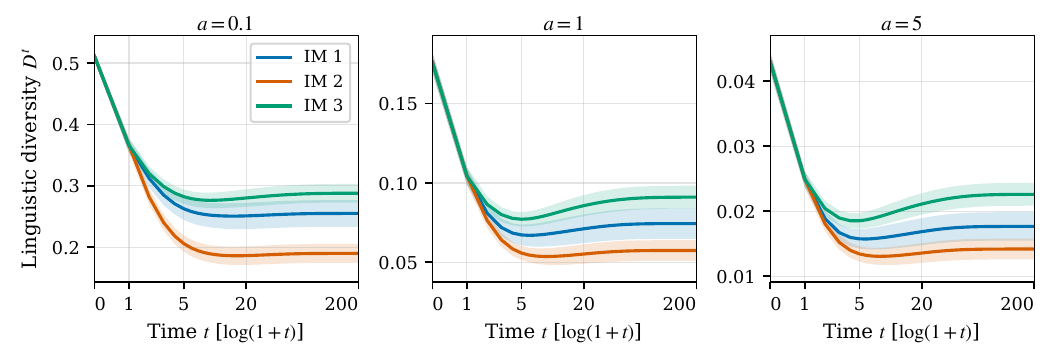}
    \caption{Linguistic-diversity trajectories $D^t$ under IM~1--3 for
$\operatorname{Dirichlet}(a,\ldots,a)$ initialization. The endpoint ordering
IM~2 $<$ IM~1 $<$ IM~3 holds at every concentration. Lines show means over
100 runs, and shading denotes $\pm1$ sample standard deviation.}
    \label{fig:appendix_robustness_dirichlet}
\end{figure}

\paragraph{Conformity heterogeneity.}
To separate positional and geometric effects, we sample
$\lambda_i\sim\operatorname{Uniform}[0.5-w,0.5+w]$ for
$w\in\{0,0.25,0.5\}$. At $w=0$, IM~1 and IM~2 have identical limiting
pairwise difference vectors, and their quadratic gap is numerically zero
($8.9\times10^{-9}$), while their JS gap is approximately $0.0071$.
The latter therefore reflects JS's dependence on location within the simplex.
As conformity becomes heterogeneous, the quadratic gap rises to approximately
$0.0041$ at $w=0.25$ and $0.0163$ at $w=0.5$, showing that heterogeneous
responses to different anchors create a genuine difference in pairwise geometry.

\begin{figure}[t]
    \centering
    \includegraphics[width=0.5\linewidth]{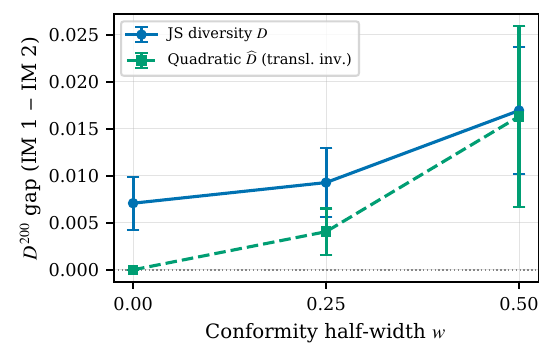}
    \caption{IM~1 minus IM~2 endpoint diversity as conformity heterogeneity
    increases. Under common conformity, the quadratic gap vanishes although
    the JS gap remains positive; with heterogeneous conformity, a genuine
    geometric gap emerges under both metrics.}
    \label{fig:appendix_robustness_conformity}
\end{figure}

\paragraph{Population-size robustness.}
Figure~\ref{fig:appendix_population_size} compares the three interaction
mechanisms for $n=100$ and $n=500$. The qualitative ordering is stable across
the two population sizes: recursive shared feedback produces the lowest
long-run diversity, while personalized feedback preserves the most diversity.
Variability across runs is smaller at $n=500$, as expected since $D^t$ averages
over $O(n^2)$ author pairs.

\begin{figure}[t]
    \centering
    \includegraphics[width=0.95\linewidth]{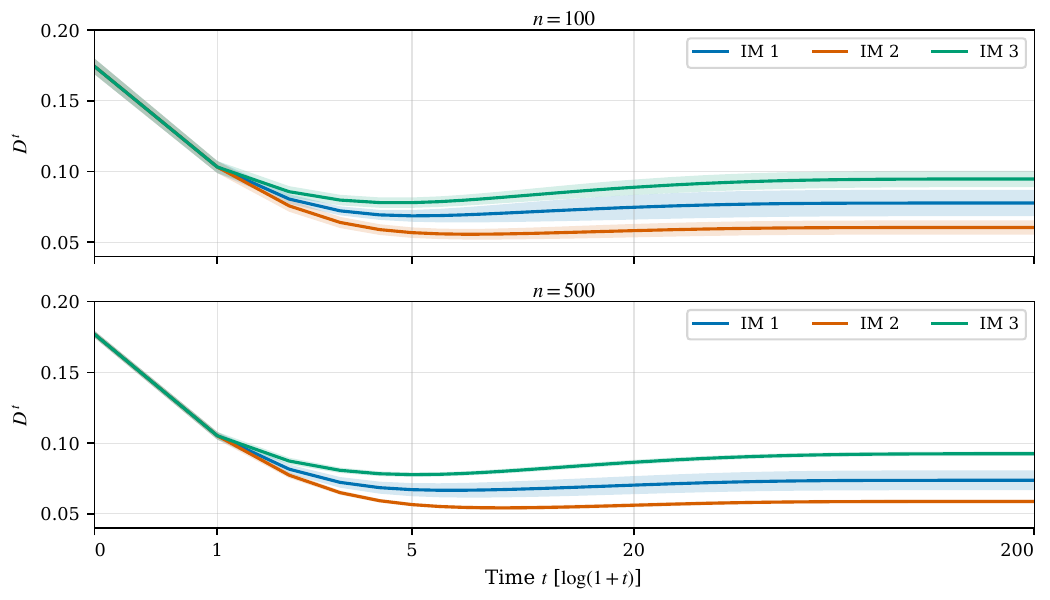}
    \caption{Population-size robustness of linguistic diversity $D^t$ under
    IM~1--3 for $n=100$ (top) and $n=500$ (bottom). Lines show means over
    100 runs and shading denotes $\pm1$ standard deviation.}
    \label{fig:appendix_population_size}
\end{figure}

\paragraph{Auxiliary convergence measures.}
Figure~\ref{fig:appendix_auxiliary_metrics} reports author--model divergence
$M^t$ and personalized-model diversity $Q^t$. Because IM~1 and IM~2 use a
single shared model, their model diversity is identically zero and is omitted
from panel~(b). Under IM~3, $M^t$ rapidly falls to a small positive level while
$Q^t$ remains positive, indicating that personalized models become distinct
while remaining closely aligned with their respective authors. The limiting model diversity is
smaller than the limiting author diversity ($Q^\star\approx0.041$ versus
$D^\star\approx0.095$), consistent with \Cref{prop:personalized_recursive}: each
$q_i^\star$ retains the author-specific component $r_i$ with weight
$\rho\eta_i<\eta_i$.

\begin{figure}[t]
    \centering
    \includegraphics[width=0.95\linewidth]{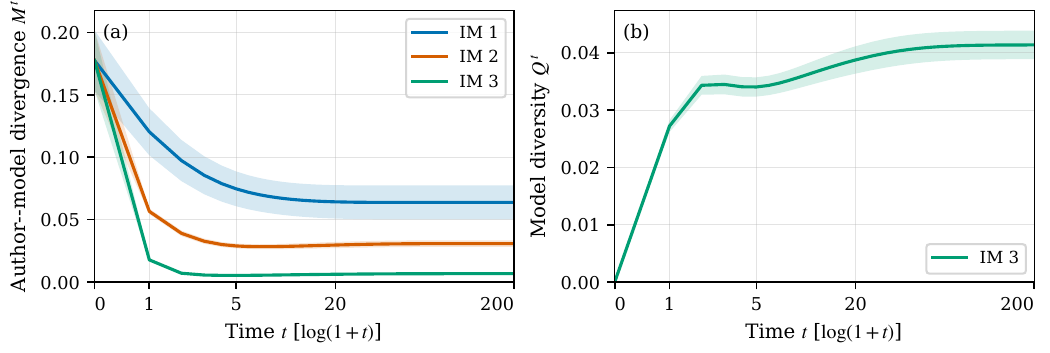}
    \caption{Auxiliary convergence measures under the baseline parameters:
    (a) author--model divergence $M^t$ for IM~1--3 and (b) personalized-model
    diversity $Q^t$ for IM~3. Lines indicate mean over 100 runs and shading
    denotes $\pm1$ standard deviation.}
    \label{fig:appendix_auxiliary_metrics}
\end{figure}

\paragraph{Personalization and transient dynamics.}
To complement the endpoint comparison in Figure~\ref{fig:linguistic_diversity},
Figure~\ref{fig:appendix_personalization} shows the full IM~3 trajectory for
$\rho\in\{0,2/3,1\}$. We hold $\beta=0.25$ fixed and set
$\gamma=(1-\beta)\rho$ and $\delta=(1-\beta)(1-\rho)$. Increasing $\rho$
therefore reallocates feedback weight from the shared population to the
associated author.

The endpoints provide a check on \Cref{prop:personalized_recursive}. At $\rho=0$ the
personalized update reduces exactly to IM~2, and the observed
$D^{200}\approx0.057$ matches the IM~2 plateau. At $\rho=1$ we have $\eta_i=1$,
so $p_i^\star=q_i^\star=r_i$; since $r_i$ and $p_i^0$ are drawn from the same
Dirichlet, diversity should return to approximately its initial level, and the
observed $0.171$ is consistent with $D^0\approx0.175$.

All mechanisms exhibit a non-monotone transient: $D^t$ undershoots its limit
near $t\approx5$ before partially recovering. Because personalized models are
initialized at the common base distribution $q^0$, every author is initially
pulled toward the same point; the $q_i^t$ differentiate only later, allowing
authors to drift back toward their preferred $r_i$. The recovery scales with
$\rho$, consistent with this account.

\begin{figure}[t]
    \centering
    \includegraphics[width=0.95\linewidth]{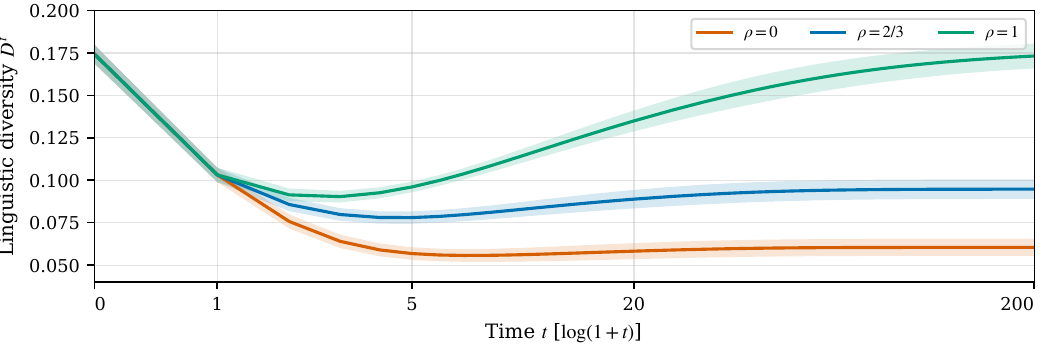}
    \caption{IM~3 linguistic-diversity trajectories for three personalization
    levels. Larger $\rho$ preserves more author-specific feedback and produces
    higher long-run $D^t$. Lines show mean over $100$ runs and shading denotes
    $\pm1$ standard deviation.}
    \label{fig:appendix_personalization}
\end{figure}

\paragraph{Heterogeneous interaction mechanisms.}
For IM~4, we assign proportions $0.33$, $0.33$, and $0.34$ of the population
to IM~1, IM~2, and IM~3, respectively, using assignment seed
$123456789$. Each subgroup follows its own update rule. In particular,
the population-level component received by personalized models is computed
only from authors in the IM~3 subgroup, as specified in \Cref{app:im4}; there is
no additional cross-subpopulation feedback. All three quantities stabilize; the
positive limiting value of $Q^t$ reflects differences among the author-facing
models used across and within the three subgroups.

\begin{figure}[t]
    \centering
    \includegraphics[width=0.98\linewidth]{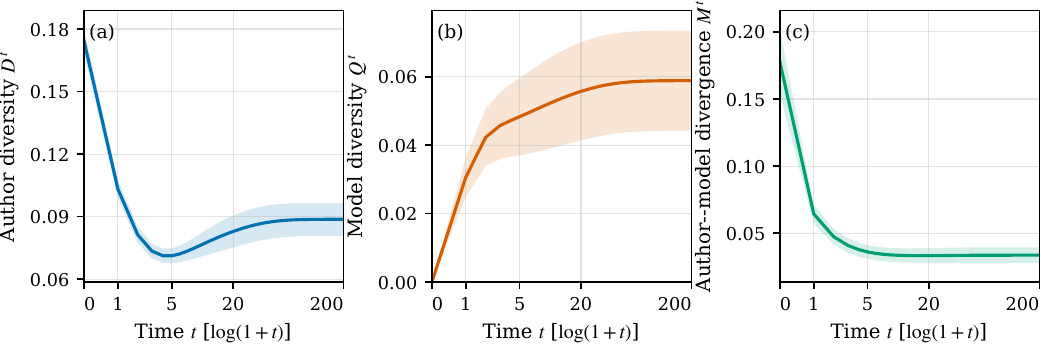}
    \caption{Evolution of (a) author diversity $D^t$, (b) diversity among
    author-facing model distributions $Q^t$, and (c) author--model divergence
    $M^t$ under heterogeneous IM~4. Lines show mean over 100 runs and
    shading denotes $\pm1$ standard deviation.}
    \label{fig:appendix_im4_metrics}
\end{figure}

\end{document}